\documentclass[3p,times]{elsarticle}

\usepackage{ecrc}

\usepackage{lineno,hyperref}

\volume{00}

\firstpage{1}

\journal{Artificial Intelligence Journal}

\usepackage[ruled,vlined]{algorithm2e}
\usepackage{amsmath}
\usepackage{amssymb}
\usepackage{amsthm}
\usepackage{arydshln}
\usepackage{listings}
\usepackage{subcaption}
\usepackage{multirow}
\usepackage{multicol}
\usepackage{listings}
\usepackage{listings}
\usepackage{varwidth}
\usepackage{tikz}
\usetikzlibrary{arrows,automata,shapes}
\newtheorem{definition}{Definition}
\newtheorem{theorem}{Theorem}

\newcommand{\tup}[1]{{\langle #1 \rangle}}
\newcommand{\strips}{\textsc{Strips}}     

\begin{document}

\begin{frontmatter}

\title{Generalized Planning as Heuristic Search: A new planning search-space that leverages pointers over objects}

\author{Javier Segovia-Aguas\corref{mycorrespondingauthor}}
\cortext[mycorrespondingauthor]{Corresponding author}
\address{Universitat Pompeu Fabra}
\ead{javier.segovia@upf.edu}

\author{Sergio Jim\'enez}
\address{Universitat Polit\`ecnica de Val\`encia}
\ead{serjice@dsic.upv.es}

\author{Anders Jonsson}
\address{Universitat Pompeu Fabra}
\ead{anders.jonsson@upf.edu}

\begin{abstract}
{\em Planning as heuristic search} is one of the most successful approaches to classical planning but unfortunately, it does not extend trivially to {\em Generalized Planning} (GP). GP aims to compute algorithmic solutions that are valid for a set of classical planning instances from a given domain, even if these instances differ in the number of objects, the number of state variables, their domain size, or their initial and goal configuration. The generalization requirements of GP make it impractical to perform the state-space search that is usually implemented by heuristic planners. This paper adapts the {\em planning as heuristic search} paradigm to the generalization requirements of GP, and presents the first native heuristic search approach to GP. First, the paper introduces a new pointer-based solution space for GP that is independent of the number of classical planning instances in a GP problem and the size of those instances (i.e. the number of objects, state variables and their domain sizes). Second, the paper defines a set of evaluation and heuristic functions for guiding a combinatorial search in our new GP solution space. The computation of these evaluation and heuristic functions does not require grounding states or actions in advance. Therefore our {\em GP as heuristic search} approach can handle large sets of state variables with large numerical domains, e.g.~integers. Lastly, the paper defines an upgraded version of our novel algorithm for GP called Best-First Generalized Planning ({\sc BFGP}), that implements a {\em best-first search} in our pointer-based solution space, and that is guided by our evaluation/heuristic functions for GP. 
\end{abstract}


\end{frontmatter}

\section{Introduction}
{\em Generalized planning} (GP) addresses the representation and computation of solutions that are valid for a set of classical planning instances from a given domain~\cite{Winner03distill:learning,Levesque:GPlanning:IJCAI11,srivastava2011new,Zilberstein:Gplanning:icaps11,hu2011generalized,illanes2019generalized,jimenez2019review,frances2021learning}. In the worst case, each classical planning instance may require a completely different solution. In practice, however, many classical planning domains are known to have polynomial algorithmic solutions~\cite{fern2011first}. In other words, one can compute a single compact general solution that exploits some common structure of the different classical planning instances in a given domain. Generalized plans are then not sequences of actions, but algorithmic solutions that supplement planning actions with control-flow constructs. For example, a generalized plan that solves any classical planning instance from the {\em blocksworld} domain~\cite{slaney2001blocks} can be compactly specified as follows: {\em put all the blocks on the table and then, in a proper order, move each block to its goal placement}. This generalized plan is able to solve any {\em blocksworld} instance, no matter the actual number, or identity of the blocks, and no matter the initial and goal configuration of the blocks. Note however that the knowledge represented in a given input set of classical planning instances may not be enough to specify an algorithmic solution that solves them all. For example, instances of the classical planning {\em blocksworld} domain do not include representation {\em features} for specifying whether {\em all blocks are on the table}, or for specifying {\em the proper order for moving the blocks to their goal placements}. A big challenge in GP is then to automatically discover the representation features that are key for computing a compact and general solution for a given set of planning instances. With this regard, researchers have proposed different {\em languages} for compactly represent GP solutions, and associated algorithms for computing a GP solution in a given language.

{\em Automated planning} has not achieved the level of integration with common programming languages, like {\sc C}, {\sc Java}, or {\sc Python}, that is  achieved by other forms of problem solving, such as {\em constraint satisfaction} or {\em operational research}~\cite{schulte2010modeling,prud2014choco,ortools}. An important  reason is the low-level representations traditionally used in planning for specifying problems and solutions~\cite{geffner2003pddl, smith2008anml,sanner2010relational,barreiro2012europa,rintanen2015impact}. Given the algorithmic kind of GP solutions, GP is a promising research direction to bridge the current gap between automated planning and programming. However, most of the work on GP still inherits the \strips\ representation, in which states are represented specifying the properties and relations of a set of objects, and where actions represent object manipulations. In this work we provide a pointer-based representation for GP problems and solutions, that is closer to common programming languages, and that applies also to the object-centered problems that traditionally are addressed in the automated planning community. In addition we show that our pointer-based representation allows us to adapt the {\em planning as heuristic search} paradigm to GP: given a GP problem that comprises a finite set of classical planning instances from a given domain,  our {\em GP as heuristic search} approach implements a combinatorial search to find a program that solves the full set of input instances. With our new pointer-based representation we are able to solve challenging programming tasks that were out of reach for previous top-down GP solvers.

{\em Heuristic search} is one of the most successful approaches to classical planning~\cite{ghallab2004automated,geffner2013concise}. The winners of the {\em International Planning Competition} (IPC) are often heuristic planners~\cite{vallati:IPC08:AIM}, and the workshop on {\em Heuristics and Search for Domain-Independent Planning} (HSDIP) is one of the discussion forums with the longest tradition at the {\em International Conference on Automated Planning and Scheduling (ICAPS)}, the major international conference for research on automated planning.
Briefly, the {\em planning as heuristic search} approach addresses the computation of sequential plans as a combinatorial search in the space of the states reachable from a single given initial state. This combinatorial search is usually implemented as a forward  search, guided by heuristics that are automatically extracted from the declarative representation of the planning problem.  There is a wide range of different heuristics for classical planning, but most of them are based on the notion of {\em relaxed plan}~\cite{helmert2009landmarks}.  The relaxed plan is a solution to a relaxation of the classical planning problem, which is computed assuming that goals (and action preconditions) are independent. The cost of the relaxed plan is an informative estimate of the actual cost-to-go for many classical planning problems, and its computation is much cheaper than the computation of the actual solution to the planning problem. 

In the last two decades a wide landscape of effective search algorithms, and heuristic functions, have been developed for classical planning~\cite{mcdermott1996heuristic,bonet2001planning,hoffmann2001ff,Helmert:FD:JAIR06,richter2010lama,lipovetzky2017best}. Unfortunately, search algorithms and heuristic functions from classical planning cannot be directly extended to GP. The computation of {\em relaxed plans}, as it is implemented by off-the-shelf heuristic planners, requires a pre-processing step for grounding states and actions~\cite{frances2015modeling}. On the other hand, GP solutions must be able to generalize to (possibly infinite) sets of classical planning instances, with different sets of state variables (i.e. state variables with different domain sizes and/or different number of state variables) as well as with different initial states and goals. These particular generalization requirements of GP make it impractical to ground states and actions and hence, to apply the state-space search and the cost-to-go estimates of heuristic planners.

With respect to previous work on GP, our heuristic search approach to GP introduces the following contributions:
\begin{enumerate}
    \item {\em A pointer-based representation for GP problems and solutions}. Our representation formalism is closer to common programming languages, and it also applies to  object-centered representations (like \strips) that are traditionally used in automated planning.
	\item {\em A tractable solution-space for GP}. We leverage the computational models of the {\em Random-Access Machine}~\cite{skiena1998algorithm} and the {\em Intel x86} FLAGS register~\cite{dandamudi2005installing} to define an innovative pointer-based search space for GP. Interestingly our new search space for GP is independent of the number of input planning instances in a GP problem, and the size of these instances (i.e. the number of objects, state variables, and their domain sizes).
	\item {\em Grounding-free evaluation/heuristic functions} for GP. We define several evaluation and heuristic functions to guide a combinatorial search in our solution space for GP.  Evaluating these functions does not require to ground states/actions in advance, so they apply to GP problems where state variables have large domains (e.g.~integers). 
	\item {\em A heuristic search algorithm for GP}. We present the {\sc BFGP} algorithm for GP that implements a best-first search in our GP solution-space, and that is guided by our evaluation and heuristic functions. 
	\item {\em A translator from the \strips\ fragment of PDDL to our pointer-based representation} for GP. We automate the representation change from PDDL to pointer-based, and show several solutions to planning domains from the {\em International Planning Competition} (IPC) \cite{vallati:IPC:AIM2015} which are validated on large random instances.
\end{enumerate}

A preliminary description of our {\em GP as heuristic search} approach previously appeared at an ICAPS conference paper~\cite{javi:GP:ICAPS21}. In this work we review and extend the seminal ideas presented in the conference paper, and provide a more exhaustive evaluation of our {\em GP as heuristic search} approach. Compared to the conference paper, the present paper includes the following novel material:
\begin{itemize}
    \item We formalize the notion of {\em pointer} over the objects 
    of a planning problem, and introduce a pointer-based formalization for planning action schemes, classical planning problems and solutions. We show that our pointer-based formalization directly applies to  object-centered planning problems that are traditionally addressed in automated planning. 
    \item We introduce the notion of {\em partially specified planning program}, that refers to the sketch of an algorithmic planning solution, and that allows to explain better our search algorithm and heuristics functions for GP. We also implemented new evaluation functions for guiding our {\em GP as heuristic search} approach.
    \item We provide theoretical results  of our heuristic search algorithm for GP, that include {\em termination}, {\em soundness}, {\em completeness}, and {\em complexity} proofs. We also extend the empirical evaluation, including more results at a wider landscape of planning domains, to characterize better the performance of our {\em GP as heuristic search} approach.
\end{itemize}

The paper is structured as follows:  Section~\ref{sec:background} presents the planning models we rely on in this work (namely the {\em classical planning} model and the {\em GP planning} model) and also presents {\em planning programs} and the {\em Random Access Machine},  the formalisms we leverage for the representation of our algorithmic planning solutions. Section~\ref{sec:pointers} shows how to extend the classical planning model with a set of pointers over objects, and the corresponding primitive operations for manipulating these pointers. This extension allows us to define, in an agnostic manner, a set of features and a set of actions for computing planning programs that can solve any  instance from a given planning domain. Section~\ref{sec:hsearch} describes our {\em GP as heuristic search} approach; the section provides details on our solution space, evaluation functions, and heuristic search algorithm for GP.  Section~\ref{sec:evaluation} presents the empirical evaluation of our approach and its comparison with the classical planning compilation for GP, that serves as a baseline. Finally, Section~\ref{sec:conclusions} wraps-up our work and discusses open issues and future work.

\subsection{Related Work}
Here we first review previous work on GP according to the following three dimensions: {\em problem representation}, {\em solution representation}, and {\em computational approach}. Then, we connect the research work on GP with other relevant areas in AI, such as {\em program synthesis}, {\em deep learning}, and {\em (deep) reinforcement learning}. Last, we list the features that distinguish our {\em GP as heuristic search approach} from  the reviewed related work.

Regarding {\em problem representation}, there are two different approaches for the specification of the set of classical planning instances that are comprised in a GP problem. The {\em explicit} approach, that enumerates every classical planning instance in a GP problem~\cite{segovia2019computing}, and the {\em implicit} approach, that defines the constraints that hold for the set of classical planning instances of a GP problem. The implicit approach is of interest because it allows to compactly specify infinite sets of classical planning instances (e.g. the infinite set of the classical planning instances that belong to the {\em blocksworld} domain)~\cite{khardon1999learning,martin2004learning,illanes2019generalized}. In addition to the set of classical planning instances, extra background knowledge can also be specified in a GP problem with the aim of reducing the space of  hypothetical solutions. For instance, {\em plan traces}/{\em demonstrations} on how to solve some of the  input instances~\cite{yoon2008learning,de2011scaling,silver2020few}, the full state space~\cite{frances2021learning}, the particular subset of state {\em features} that can be used for computing a generalized plan~\cite{Geffner:FSM:AAAI10,blai:gplanning:ijcai18}, {\em negative examples} that specify undesired behavior for the targeted GP solutions~\cite{aguas2020generalized,frances2021learning}, or {\em state invariants} that any state in a given domain must satisfy~\cite{lotinac2016constructing}. 

With respect to {\em solution representation}, different formalisms appeared in the planning literature to represent  solutions that are valid for a set of classical planning instances; sequential plans are used in {\em conformant planning}~\cite{palacios2009compiling},  conditional {\em tree-like} plans are used in {\em contingent planning}~\cite{geffner2013concise}, or policies are used in {\em FOND planning}, as well as in MDP/POMDP planning~\cite{kolobov2012planning}. In all these planning settings, a set of different classical planning instances, with different initial states, can be implicitly represented as a disjunctive formulae over the state variables. Different goals can also be considered by coding them as part of the state representation, e.g. using {\em static} state variables. Since the early days of AI planning, hierarchies, LTL formulae, and policies, are also used to specify {\em sketches} of general solutions~\cite{jimenez2012review}. In the planning literature these solution sketches are often called {\em domain-specific control knowledge}, since they are traditionally used to control the planning process, and they apply to the entire set of classical planning instances that belong to a given   domain~\cite{bacchus2000using,nau2001shop,yoon2008learning,de2011scaling}. Last but not least, algorithmic solutions, represented either as lifted policies, finite automata, or as programs with control-flow constructs for branching and looping, are used to represent GP solutions~\cite{khardon1999learning,Geffner:Gpolicies:AppliedI04,Geffner:FSM:AAAI10,javi-Gplanning-IJCAI16,ramirez2016heuristics,segovia2018computing,Winner03distill:learning,segovia2019computing,jimenez2019review}. 

Regarding to the {\em computation of generalized plans}, there are two main approaches for addressing GP problems. The {\em top-down/offline} approach considers the entire set of classical planning instances in a given GP problem as a single batch, and computes a solution plan that is valid for the full batch at once. A common approach for the offline computation of generalized plans is compiling the GP problem into another form of problem solving, and using an off-the-shelf solver to work out the compiled problem. With this regard, GP problems have been compiled into classical planning problems~\cite{segovia2018computing,segovia2019computing}, conformant planning problems~\cite{Geffner:FSM:AAAI10}, LTL synthesis problems~\cite{bonet2020high}, FOND planning problems~\cite{bonet2019learning,illanes2019generalized} or MAXSAT problems~\cite{frances2021learning}. The compilation approach is appealing because it allows to leverage the latest advances of other well-founded scientific communities, with robust and scalable solvers. In addition, the computational complexity of some of these tasks is theoretically characterized with respect to structural features of the input problems, which may provide insights on the difficulty of the addressed GP problem.  A weak point of the compilation approach is however the size of the compiled problems to be solved; solvers are usually sensitive to the size of the input problems. On the other hand, the {\em bottom-up/online} approach incrementally processes the set of classical planning instances in a GP problem~\cite{Winner03distill:learning,srivastava2011new}. Given a classical planning instance, a solution to that instance is computed and then, the solution is merged with solutions computed for the previous instances.   The online approach is then appealing for handling GP problems that comprise large sets of classical planning instances. The main drawback of online approaches is dealing with the overfitting produced by the individual processing of the different classical planning instances in a GP problem.

As noted by previous work on GP, the aims of GP are connected to {\em program synthesis}~\cite{segovia2019computing,illanes2019generalized,bonet2020high,javi:GP:ICAPS21}. Program synthesis is a task traditionally studied by the {\em computer-aided verification} community~\cite{kurshan2014computer}, and that aims the  computation of programs such that they satisfy a given correctness specification~\cite{solar2008program,gulwani2017program,alur2018search}. Program synthesis follows the {\em functional programming} paradigm. This means that a program is a function composition, where each function in the composition is a mapping of its input parameters into a single output, and where loops are implemented using recursion.  Work on program synthesis is classified according to how the correctness specification of a program is formulated. The {\em programming by example} (PbE) paradigm specifies the desired program behaviour with a finite and non-empty set of ground {\em input/output} examples. This approach is related to the explicit representation of GP problems; a ground {\em input/output} example can be understood as the {\em initial/goal} state pair that represents a classical planning instance, and the instruction set of the functional programming language can be understood as the available actions for transforming an initial state into a goal state. Program synthesis also allows the implicit representation of the input correctness specifications, e.g. using fist-order formulae specified in SMTLIB, the formal language for SAT-Modulo Theories (SMT)~\cite{barrett:smtlib:2010}. The mainstream approach for program synthesis is to specify a formal grammar that allows to incrementally enumerate the  space of possible programs, and to leverage the satisfiability machinery of SMT solvers to validate whether a candidate program is actually a solution. With this regard, work on {\em theorem proving} is also related to program synthesis, specially since SMT solvers allow the representation and satisfaction of first-order logic formulae~\cite{loveland2016automated}. Lastly, another popular trend in program synthesis is {\em Programming by sketches} that addresses program synthesis in the particular setting where a partially specified solution is provided as input~\cite{solar2013program}.  

Besides computational methods for formal verification and logic satisfaction, optimization methods (that are predominant in {\em Machine Learning}~\cite{burkov2019hundred}) have also been applied to the computation of planning solutions that generalize. For instance, off-the-shelf {\em Deep Learning} (DL) tools, have been successfully applied to the computation of generalized policies for classical and probabilistic planning domains~\cite{toyer2018action,bueno2019deep,garg2020symbolic}. Generalized policies are a powerful solution representation formalism whose applicability goes beyond classical planning. Generalized policies can represent planning solutions that can deal with non-deterministic actions~\cite{sanner2009practical}, and whose aim is not to satisfy a given goal condition but to optimize a given utility function~\cite{fern2006approximate}.  The aims of GP are also related to {\em Reinforcement Learning} (RL)~\cite{sutton2018reinforcement}; while the cited DL approaches can be viewed as off-line optimization approaches to GP, the RL paradigm can be viewed as an online optimization approach to GP.  RL methods incrementally compute policies, by iteratively addressing a set of sequential decision-making episodes. In RL learning experience is however not given beforehand (learning experience is collected by the autonomous exploration of the state space), and RL assumes that there is an explicit notion of reward function (which helps to guide  exploration towards the most promising portions of the state-space). Note that DL and DRL approaches learn policies, without requiring a symbolic representation of the state and the action space. This means that it is possible to compute policies ({\em deep policies}) that generalize from raw sensor data (e.g.~sequences of images)~\cite{groshev2018learning,junyent2019deep}. The main disadvantage of computing solutions represented as deep policies is that they are black-box models that lack transparency and explanation capacity, which makes it difficult to interpret the produced solutions. This is a strong requirement in application areas that require humans in the loop, such as health, law, or defense~\cite{pearl2019limitations}.

With regard to the reviewed related work, our {\em GP as heuristic planning} approach is framed as follows:
\begin{itemize}
    \item {\em Numeric state variables}. Previous work on generalized planning mainly followed the object-centered \strips{} representation. Addressing programming tasks with such representation is unpractical since it requires to encode all values in the domain of a state variable as objects. Other approaches, such as Qualitative Numeric Planning (QNP) \cite{srivastava2011qualitative,bonet2019qualitative}, handle large numeric state variables qualitatively with propositions to denote whether a variable is equal to zero.
    In this work we handle GP problems with integer state variables, which allow to naturally address diverse programming tasks as if they were GP problems. 
    \item {\em Explicit problem representation}. In this work, a GP problem comprises the explicit enumeration of a finite set of classical planning instances to be solved. Interestingly our experimental results show that, in several domains, solving a small set of a few randomly generated classical planning instances, is enough to obtain a solution that generalizes to the infinite set of problems that belong to a given domain. 
    \item {\em No background knowledge}. Our approach does not require any additional help such as state invariants, plans/traces/demonstrations, negative examples, or the specification of the subset of features to appear in the generalized plans. 
    \item {\em Generalized plans represented as structured programs}. Structured programming provides a white-box modeling paradigm that is widely popular. In this work we focus on generalized plans represented as structured programs, with control flow constructs for branching and looping the program execution flow. The application of a generalized plan on a particular instance is then a deterministic matching-free process, which makes it easier to define effective evaluation and heuristic functions.  Further, the asymptotic complexity of structured programs can be assessed from their structure, which is also helpful to establish preferences on different possible generalized plans. 
    \item {\em Off-line satisfiability approach}. This work follows an off-line approach to GP that aims to compute, at once, a generalized plan that exactly solves all the classical planning instances that are given as input. Because many heuristic search algorithms are easily extended to online versions, we believe that our GP as heuristic search approach is a stepping stone towards online approaches that can deal with larger sets of classical planning instances. 
    \item {\em Native heuristic search for GP}. By {\em native} heuristic search, we mean that we defined a search space, evaluation/heuristic functions, and a search algorithm, that are specially targeted to GP. Our {\em GP as heuristic search} approach is related to an existing classical planning compilation for GP~\cite{segovia2019computing}. Our approach overcomes however the main drawback of the compilation whose search space grows exponentially with the number and domain size of the state variables; in practice, this limits the applicability of the compilation to planning instances of small size since the performance of off-the-shelf classical planners is sensitive to the size of the input instances. Our experiments support this claim, and show that our {\sc BFGP} algorithm significantly reduces the CPU-time required to compute and validate generalized plans, compared to the classical planning compilation approach to GP~\cite{segovia2019computing}. 
\end{itemize}

\section{Background}
\label{sec:background}
This section introduces the necessary notation to formalize our {\em GP as heuristic search} approach. First, the section formalizes the {\em classical planning} model and the {\em generalized planning} model. Then the section formalizes {\em planning programs}, our formalism for the compact representation of planning solutions, and that applies to both classical planning and generalized planning. Lastly the section formalizes the {\em Random Access Machine} given that, to define a tractable solution space for GP, our {\em GP as heuristic planning} approach borrows several mechanisms from this abstract computation machine.

\subsection{Classical Planning}
Our formalization of the classical planning model is  similar to the abstract planning framework called {\em Finite Functional Planning}, that was introduced for the theoretical analysis of different ground languages for classical planning~\cite{backstrom2011all}. Let $X$ be a set of {\em state variables}, where each variable $x\in X$ has a domain $D_x$. A {\em state} $s$ is a total assignment of values to the set of state variables, i.e.~$s=\tup{x_0=v_0, \ldots, x_N=v_N}$, such that $\forall_{0\leq i\leq N} v_i\in D_{x_i}$. For a subset of the state variables $X'\subseteq X$, let $D[X']=\times_{x\in X'} D_x$ denote its joint domain. The state space is denoted as $S=D[X]$. Given a state $s\in S$, and a subset of variables $X'\subseteq X$, let $s_{|X'}=\tup{x_i=v_i}_{x_i\in X'}$ be the {\em projection} of $s$ onto $X'$ i.e.~the partial state that is defined by the values assigned by $s$ to the subset of state variables in $X'$. The {\em projection} of $s$ onto $X'$ defines the subset $\{s \mid s \in S, s_{|X'}\subseteq s\}$ of the states that are consistent with the corresponding partial state. Last, let us define a {\em state-constraint} $C$ as a Boolean function $C:S\rightarrow\{0,1\}$ over the state variables, that implicitly defines the subset of states $S_C\subseteq S$ that are consistent with that constraint.

Let $A$ be a {\em set of deterministic actions} such that each action $a\in A$ is characterized by two functions; an {\em applicability function} $\rho_a: S \rightarrow \{0,1\}$ and a {\em successor function} $\theta_a: S\rightarrow S$. An action $a\in A$ is applicable in a given state $s\in S$ iff $\rho_a(s)$ equals $1$. The execution of an applicable action $a\in A$, in a state $s\in S$ results in the {\em successor} state $s'=s\oplus a$. Note that our definition of deterministic actions generalizes actions with conditional effects~\cite{nebel2000cond}, common in GP since their state-dependent outcomes allows the adaptation of generalized plans to different classical planning instances.  

A {\em classical planning instance} is a tuple $P=\tup{X,A,I,G}$, where $X$ is a set of state variables, $A$ is a set of actions, $I\in S$ is an initial state, and $G$ is a constrain on the value of the state variables that induces the subset of {\em goal states} $S_G = \{s \mid s \vDash G, s \in S\}$. Given $P$, a {\em plan} is an action sequence $\pi=\tup{a_1, \ldots, a_m}$ whose execution induces a {\em trajectory} $\tau=\tup{s_0, a_1, s_1, \ldots, a_m, s_m}$ such that, for each $1\leq i\leq m$, $a_i$ is applicable in $s_{i-1}$ and results in the successor $s_i=s_{i-1}\oplus a_i$. A plan $\pi$ {\em solves} $P$ if and only if the execution of $\pi$ in $s_0=I$ finishes in a goal state, i.e.~$s_m\in S_G$. We say $\pi$ is {\em optimal} if $|\pi|=m$ is minimal among the set of all the plans that solve $P$.

Planning languages, such as PDDL~\cite{haslum2019introduction}, can compactly represent the infinite set of classical planning instances of a given domain  using a finite set of functions and a finite set of action schemes. Given a finite set of objects $\Omega$, and a finite set of functions $\Phi$ defined over that set of objects, we assume that each state variable $x \in X$ stands for a function interpretation $x\equiv \phi(\overrightarrow{o})$, where $\phi\in\Phi$ is a function with arity $ar(\phi)$, and $\overrightarrow{o} \in \Omega^{ar(\phi)}$ is a vector of objects comprised in the Cartesian product space of $\Omega^{ar(\phi)}$. For keeping compact the number of state variables, objects and function signatures can by typed so the number of possible function interpretations is constrained. Functions in $\Phi$ can be Boolean e.g. to represent PDDL predicates, or numeric e.g. to represent PDDL numeric fluents. Likewise, given a set of action schemes $\Xi$, we assume that each action $a\in A$ is built from an action schema $\xi\in\Xi$ by substituting each variable in the action scheme with an object from $\Omega$. An action scheme $\xi\in\Xi$ is a tuple $\xi=\tup{name(\xi), par(\xi), pre(\xi), eff(\xi)}$; where $name(\xi)$ is the identifier of the action schema, $par(\xi)$ is its list of variables (again these variables can be typed so they can only be substituted by objects of the same type), $pre(\xi)$ is a FOL Boolean formula  defined over $par(\xi)$ that compactly represents the subset of states where the corresponding ground actions are applicable, and {\em eff}$(\xi)$ is list of FOL constraints that compactly represents the updates of the state variables caused by the application of the corresponding ground actions.

\subsection{Generalized Planning}
{\em Generalized planning} is an umbrella term that refers to more general notions of planning~\cite{jimenez2019review}. This work builds on top of the inductive formalism for GP, where a GP problem is defined as a finite set of classical planning instances that share a common structure~\cite{hu2011generalized,bonet2020high}. In this work we assume that the finite set of classical planning instances in a GP problem belong to the same domain. These instances share then a common structure since they are built from the same sets of functions $\Phi$ and action schemes $\Xi$. 
\begin{definition}[GP problem]
	\label{def:gp-problem}
	A {\em GP problem} $\mathcal{P}=\{P_1,\ldots,P_T\}$ is a finite and non-empty set of $T$ classical planning instances  $P_1=\tup{X_1,A_1,I_1,G_1}, \ldots, P_T=\tup{X_T,A_T,I_T,G_T}$ such that at each instance $P_t\in\mathcal{P}$,  {\small $1\leq t \leq T$}, may differ in the set of state variables, actions, initial state, and goals, but the corresponding set of state variables $X_t$ is induced from the common set of functions $\Phi$, and the set of actions $A_t$ from the common set of action schemes $\Xi$.
\end{definition}

There are diverse representations for GP solutions that range from {\it generalized polices}~\cite{khardon1999learning,Geffner:Gpolicies:AppliedI04}, to {\em finite state controllers}~\cite{Geffner:FSM:AAAI10,javi-Gplanning-IJCAI16}, formal grammars~\cite{ramirez2016heuristics},  hierarchies~\cite{nau:shop2:JAIR03,segovia2018computing}, or programs~\cite{Winner03distill:learning,segovia2019computing}. Each representation has its own expressiveness capacity, as well as its own validation complexity and computation complexity. In spite of this representation diversity, we can define a common condition under which a generalized plan is considered a solution to a GP problem. First, let us define $exec(\Pi,P)=\tup{a_1, \ldots, a_m}$ as the sequential plan that is produced by the execution of a {\em generalized plan} $\Pi$ on a classical planning instance $P$.

\begin{definition}[GP solution]
	\label{def:gp-solution}
	A {\em generalized plan} $\Pi$ solves a GP problem $\mathcal{P}=\{P_1,\ldots,P_T\}$ iff, for every classical planning instance $P_t\in \mathcal{P}$, $ 1\leq t\leq T$, it holds that the sequential plan $exec(\Pi,P_t)$ solves $P_t$.
\end{definition}

A GP solution $\Pi$ for a given {\em GP problem} $\mathcal{P}$ is optimal iff, for every $P_t\in\mathcal{P}$, the sequential plan $exec(\Pi,P_t)$ induced by $\Pi$ for solving the classical planing instance $P_t$ is an optimal plan for that instance.

\vspace{.5cm}

{\bf Example}. Figure~\ref{fig:cp-example} shows the initial state and goal of two classical planning instances, $P_1=\tup{X,A,I_1,G_1}$ and $P_2=\tup{X,A,I_2,G_2}$, for sorting two six-element lists. In this example the two instances share the same set of state variables $X=\{x_i\equiv vector(o_i) | 0 \leq i \leq 5 \}$ that is built with the one-arity function $\Phi=\{vector\}$ and the set of objects $\Omega_1 = \Omega_2 = \{o_0,\ldots,o_5\}$, and where $\forall_{x\in X} D_x = \mathbf{N}_0$. The two classical planning instances also share the set of deterministic actions $A$, with $\frac{6\times 5}{2}$ actions $swap(o_i,o_j)$, that swap the content of two list positions $i<j$, and that are induced from the single action scheme $\Xi=\{swap(x,y)\}$. An example solution plan for $P_1$ is  $\pi_1=\tup{swap(o_0,o_5), swap(o_1,o_2), swap(o_1,o_3)}$ while $\pi_2=\tup{swap(o_0,o_2), swap(o_3,o_5)}$ is an example of a sequential plan that solves $P_2$. Note that the set $\mathcal{P}=\{P_1,P_2\}$ is a GP problem since they are two classical planning instances that are built using the same set of functions $\Phi$ and action schemes $\Xi$. Figure~\ref{fig:sortnet} shows an example of a generalized plan that solves the GP problem  $\mathcal{P}=\{P_1,P_2\}$, and that is represented as a {\em sorting network}.  The sorting network is illustrated using two different types of items (namely the {\em wires} and the {\em comparators}). For each state variable, there is a wire that carries the value of that variable from left to right in the network. On the other hand, comparators connect two different wires, corresponding to a pair of variables $(x_i,x_j)$, such that $i<j$. When a pair of values traveling through a pair of wires $(i,j)$, encounters a comparator, then the comparator applies the action $swap(o_i,o_j)$ iff $vector(o_i)\geq vector(o_j)$, which in turn is $x_i \geq x_j$. The sorting network of Figure~\ref{fig:sortnet} can actually solve any instance for sorting the content of any six-element list, no matter its initial content. This  solution is however not valid for sorting lists with different lengths. In this paper we will show how to represent and compute planning solutions that leverage {\em indirect memory addressing} to generalize no matter the number of objects, and corresponding state variables.

\begin{figure}
	\centering
	\begin{tikzpicture}
	\draw[draw=black,step=0.5cm] (0.0,0.0) grid (3.0,0.5);
	\draw[draw=black,step=0.5cm] (4.0,0.0) grid (7.0,0.5);
	\draw[draw=black] (4.0,0.0) -- (4.0,0.5);
	\draw[draw=black,step=0.5cm] (0.0,1.0) grid (3.0,1.5);
	\draw[draw=black,step=0.5cm] (4.0,1.0) grid (7.0,1.5);
	\draw[draw=black] (4.0,1.0) -- (4.0,1.5);
	\draw[draw=black] (0.0,1.0) -- (3.0,1.0);
	\draw[draw=black] (4.0,1.0) -- (7.0,1.0);
	\draw[->,black] (3.1,0.25) -- (3.9,0.25);
	\draw[->,black] (3.1,1.25) -- (3.9,1.25);
	\node at (1.5,2.0) {\textbf{Initial State}};
	\node at (5.5,2.0) {\textbf{Goal State}};
	\node at (-0.5,0.25) {$P_2$};    
	\node at (0.25,0.25) {3};
	\node at (0.75,0.25) {2};
	\node at (1.25,0.25) {1};
	\node at (1.75,0.25) {6};
	\node at (2.25,0.25) {5};
	\node at (2.75,0.25) {4};
	\node at (4.25,0.25) {1};
	\node at (4.75,0.25) {2};
	\node at (5.25,0.25) {3};
	\node at (5.75,0.25) {4};
	\node at (6.25,0.25) {5};
	\node at (6.75,0.25) {6};
	\node at (-0.5,1.25) {$P_1$};        
	\node at (0.25,1.25) {6};
	\node at (0.75,1.25) {3};
	\node at (1.25,1.25) {4};
	\node at (1.75,1.25) {2};
	\node at (2.25,1.25) {5};
	\node at (2.75,1.25) {1};
	\node at (4.25,1.25) {1};
	\node at (4.75,1.25) {2};
	\node at (5.25,1.25) {3};
	\node at (5.75,1.25) {4};
	\node at (6.25,1.25) {5};
	\node at (6.75,1.25) {6};
	\end{tikzpicture}
	\caption{\small Example of two classical planning instances for sorting the content of two six-element lists by swapping the list elements.}
	\label{fig:cp-example}
	\begin{tikzpicture}
		\node at (0,2.8) {\small $x_5$};
	    \node at (0,2.3) {\small $x_4$};
		\node at (0,1.8) {\small $x_3$};		
	    \node at (0,1.3) {\small $x_2$};
		\node at (0,0.8) {\small $x_1$};	
	    \node at (0,0.3) {\small $x_0$};
		\node at (0,0) {\small };
	\end{tikzpicture}
\includegraphics[scale=.5]{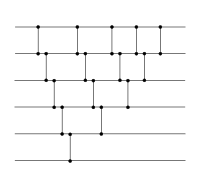}
	\caption{\small Example of a generalized plan, represented as a {\em sorting network} that solves any classical planning instances for sorting the content of a six-element list, no matter its initial content.}
	\label{fig:sortnet}
\end{figure}

\newpage
\subsection{Planning programs}
In this work we represent planning solutions as {\em planning programs}~\cite{segovia2019computing}. Unlike sequential plans, {\em planning programs} include a control flow construct which allows the compact representation of solutions to classical planning problems and to GP problems.  Formally a {\em planning program} is a sequence of $n$ instructions  $\Pi=\tup{w_0,\ldots,w_{n-1}}$, where each instruction $w_i\in \Pi$ is associated with a {\em program line} {\small $0\leq i< n$}, and it is either: 
\begin{itemize}
	\item A {\em planning action} $w_i\in A$.
	\item A {\em goto instruction} $w_i=\mathsf{go}(i',!y)$, where $i'$ is a program line $0\leq i'<i$ or $i+1<i'< n$, and $y$ is a proposition.
	\item A {\em termination instruction} $w_i=\mathsf{end}$. The last instruction of a planning program is always a termination instruction, i.e. $w_{n-1}=\mathsf{end}$. 
\end{itemize}

The execution model for a planning program  is a {\em program state} $(s,i)$, i.e.~a pair of a planning state $s\in S$ and program counter $0\leq i<n$. Given a program state $(s,i)$, the execution of a programmed instruction $w_i$ is defined as:
\begin{itemize}
	\item If $w_i\in A$, the new program state is $(s',i+1)$, where $s'=s\oplus w_i$ is the {\em successor} when applying $w_i$ in $s$.
	\item If $w_i=\mathsf{go}(i',!y)$, the new program state is $(s,i+1)$ if $y$ holds in $s$, and $(s,i')$ otherwise\footnote{We adopt the convention of jumping to line $i'$ whenever $y$ is {\em false}, following the JMP instructions in the {\em Random-Access Machine} that jump when a register equals zero.}. Proposition $y$ can be the result of an arbitrary expression on state variables, e.g.~a state {\em feature}~\cite{lotinac2016automatic}. 
	\item If $w_i=\mathsf{end}$, program execution terminates. 
\end{itemize}

To execute a planning program $\Pi$ on a classical planning instance $P=\tup{X,A,I,G}$, the initial program state is set to $(I,0)$, i.e.~the initial state of $P$ and the first program line of $\Pi$. A program $\Pi$ {\em solves} $P$ iff the execution terminates in a program state $(s,i)$ that satisfies the goal condition, i.e.~$w_i=\mathsf{end}$ and $s\in S_G$. Otherwise the execution of the program fails. If a planning program {\em fails} to solve the planning instance, the only possible sources of failure are: 
\begin{enumerate}
    \item {\em Inapplicable program}, i.e.~executing action $w_i\in A$ fails in program state $(s,i)$ since $w_i$ is not applicable in $s$.
    \item {\em Incorrect program}, i.e.~execution terminates in a program state $(s,i)$ that does not satisfy the goal condition, i.e.~($w_i=\mathsf{end})\wedge (s\notin S_G)$.
\item {\em Infinite program}, i.e.~execution enters into an infinite loop that never reaches an $\mathsf{end}$ instruction.
\end{enumerate}

In this work we model instructions $w_i\in A$ as if they were always applicable but that their effects only update the current state iff the preconditions of the action hold in the current planning state. Formally, when executing $w_i$ in $(s,i)$, the new program state is $(s',i+1)$ iff $w_i$ is applicable, otherwise it is $(s,i+1)$. Therefore, in this work the execution of a program on a classical planning instance will never return an {\em inapplicable program}, and only {\em incorrect} or {\em infinite program} are possible sources of failure. This particular action modeling is common in {\em Reinforcement Learning}~\cite{sutton2018reinforcement}, and in {\em conformant planning}~\cite{palacios2009compiling}, because it allows to deliver compact solutions that apply to sets of different sequential decision-making problems (typically with different initial states).

\subsection{The Random-Access Machine}
The {\em Random-Access Machine} (RAM) is an abstract computation machine, in the class of {\em register machines}, that is polynomially equivalent to a Turing machine~\cite{boolos2002computability}. The RAM machine enhances a multiple-register {\em counter machine}~\cite{minsky1961recursive} with indirect memory addressing. The indirect memory addressing of RAM machines is useful for defining RAM programs that access an unbounded number of registers, no matter how many there are. With this regard, a {\em register} in a RAM machine is then a memory location with both a {\em content} (a single natural number), and an {\em address} (a unique identifier that works as a natural number). Let be $r$ the address of a RAM register, and $[r]$ the content of that register. 

Diverse {\em base instructions sets}, that are Turing complete, can be defined on the RAM registers. Our {\em GP as heuristic search} approach builds on the instructions of the {\em Base set 1} and the {\em Base set 3}, that are briefly defined as follows:
\begin{itemize}
	\item $inc(r)$ increments the content of the register $r$ by $1$, i.e. $([r]+1)\rightarrow [r]$. Likewise, $dec(r)$ decrements the content of the register $r$, i.e. $([r]-1) \rightarrow [r]$.
	\item $copy(r_1,r_2)$ copies the content of $r_1$ in $r_2$, i.e. $[r_1]\rightarrow [r_2]$.
	\item $jmp(z,r)$ jumps to instruction $z$ if register $r$ is zero, else the RAM execution continues to the next instruction. 
	\item $halt$ terminates the RAM execution.
\end{itemize}

Auxiliary {\em dedicated} registers allow to reduce the size of the RAM instruction set. For instance, the number of registers  called out by the RAM instructions can be reduced  by using an {\em accumulator} dedicated register. WLOG in this work we extend {\em Base sets 1} and {\em 3} with two FLAGS registers, the $zero$ and the $carry$ registers, that are dedicated to store the outcome of {\em three-way comparisons}~\cite{browning2020working} between two registers, as well as to condition jump instructions. On the other hand, the RAM instruction set can be extended with extra instructions, that can actually be built as blocks of the base RAM instructions, but that allow the definition of more compact RAM programs. For instance, an extra instruction for {\em clearing} (setting to zero) the value of a register.

\section{Planning with a RAM}
\label{sec:pointers}
Inspired by the computational model of the RAM machine, this section extends the classical planning model with a set of {\em pointers}, defined over the set of objects used to build the state variables and actions of a classical planning instance, and with the set of {\em primitive operations} for  manipulating these pointers. Extending the classical planning model with a set of pointers, and their primitives, allows the agnostic definition of a set of state features, and a set of actions, that are shared for the different classical planning instances in a given domain and that can be leveraged for the computation of generalized plans. For instance, the {\em selection sort} algorithm is able to solve any sorting instance, no matter the content or the length of the input list, because it is provided with mechanisms for indexing the list positions and for operating over those indexes. 

First, the section shows how to compactly represent a transition system using pointers. Then the section formalizes our extension of the classical planning model with a set of {\em pointers}, and their corresponding primitive operations. The last part of the section shows that our pointer-based formalism applies also to the object-centered representations traditionally used in classical planning.

\subsection{Representing transition systems with pointers}
Formally, a {\em transition system} is a pair $(S,\rightarrow)$, where $S$ is a set of states, and $\rightarrow$ denotes a relation of state transitions $S\times S$. Unlike finite-state automata, the set of states and the set of state transitions of a transition system are not necessarily finite, and no initial/goal states are specified. {\em State constraints} allow the compact representation of (possibly infinite) sets of states. For instance given the set of  state variables $X=\{x_0, x_1, x_2, x_3 ,x_4, x_5\}$ of the example in Figures~\ref{fig:cp-example} and~\ref{fig:sortnet}, the state constraint $x_0\leq x_1\leq x_2\leq x_3\leq x_4\leq x_5$ defines the subset of states where the content of these variables is sorted in increasing order and it applies no matter the domain of the state variables (that can actually be infinite). 

In {\em object-centered} transition systems, states are factored and each state variable is addressed by a function $\phi \in \Phi$ fed with a list of objects $\overrightarrow{o}$, i.e. $\phi(\overrightarrow{o})$ s.t. $\overrightarrow{o}\in \Omega^{ar(\phi)}$; where $\Omega$ is the set of objects and $ar(\phi)$ denotes the arity of  $\phi$. For instance in the previous example, given the one-arity function $vector$ and the six-objects set $\Omega = \{o_0, o_1, o_2, o_3, o_4, o_5 \}$, each state variable $x_i \in X$ may be defined as $x_i \equiv vector(o_i)$. 
\begin{definition}[Pointer]
Given a set of objects $\Omega$, we define a pointer as a finite-domain variable $z$ whose domain is $D_z=[0,\ldots,|\Omega|-1]$, where $|\Omega|$ denotes the number of objects.
\end{definition}

{\em Pointers} are then variables for indexing the objects of a transition system that, in combination with function symbols, are useful to define state constraints that produce not only compact, but {\em general} representations of a possibly infinite set of states. By general we mean that a constraint represents a set of states that share some common structure, no matter the actual number of objects, and corresponding state variables. For instance, Figure~\ref{fig:gconstraint} shows the Boolean function {\tt\small constraint\_sorted}, that implements a {\em global constraint} for checking whether the content of the vector of state variables is sorted in increasing order. The {\tt\small constraint\_sorted} function is procedurally defined, leveraging a single pointer $i$, and it applies to any number of objects, and to any domain size of the corresponding state variables. 

\begin{figure}
\centering
\begin{tabular}{c}
\begin{small}
\begin{lstlisting}[basicstyle=\ttfamily,columns=fullflexible,keepspaces=true]
Bool constraint_sorted () {	
    For (Pointer i:=1; i<|$\Omega$|; i++) {
        If ( vector(i-1) > vector(i) ) 
            Return False;
    }
    Return True;
}
\end{lstlisting}
\end{small}
\end{tabular}
	\caption{\small Boolean function {\tt\small constraint\_sorted} that implements a   constraint for validating whether the vector of state variables is sorted in increasing order. The constraint is implemented leveraging the single pointer $i$ over the objects in $\Omega$; $vector(i)$ is interpreted as $vector(o_i)\equiv x_i \in X$.}
	\label{fig:gconstraint}
\end{figure}

Besides the compact and general definition of (possibly infinite) sets of states, pointers over objects also enable the compact and general definition of (possibly infinite) sets of state transitions via {\em action schemes}. 
\begin{definition}[Action schema with pointers] Given a set of $X$ state variables, an {\em action schema} with pointers is a tuple $\langle${\it name, params, pre, eff}$\rangle$ where:
\begin{itemize}
    \item {\em name} is the symbol that uniquely identifies the action schema.
    \item {\em params} is a finite set of pointers $Z$ defined over the set $\Omega$ of objects.
    \item {\em pre} is a state constraint where state variables are indirectly addressed via the function symbols and the pointers in {\em params}, i.e. $x\equiv \phi(\overrightarrow{z})$ such that $\phi\in \Phi$ and $\overrightarrow{z}\in Z^{ar(\phi)}$. The {\em pre} state constraint implicitly represents the subset of states where the action schema is applicable.
    \item {\em eff} is a partial assignment of the state variables where state variables are indirectly addressed via the function symbols and the pointers in {\em params}. The {\em eff} partial assignment implicitly represents the successor state that results from the execution of the action schema at a given state. 
\end{itemize}
\end{definition}
To illustrate our pointer-based definition of an action schema, Figure~\ref{fig:ascheme} shows a procedural representation of the preconditions ({\em pre}) and the  effects ({\em eff}) of the {\tt\small swap} action schema. When applicable, the {\tt\small swap} action schema exchanges the value of the state variables induced by its two  parameters (the pointers $i$ and $j$). The {\tt\small swap} action schema is succinct, because it compactly defines an infinite set of different state transitions that share a common structure. The {\tt\small swap} action schema is also general, because it applies to any sorting instance, no matter the number of state variables (i.e. the length of the vector of state variables) or the domain size of these variables. What is more, the execution of the {\tt\small schema\_swap\_pre} and {\tt\small schema\_swap\_eff} procedures is a deterministic matching-free process since the input pointers do always index an object in $\Omega$.

\begin{figure}
\centering
\begin{tabular}{c}
\begin{small}
\begin{lstlisting}[basicstyle=\ttfamily,columns=fullflexible,keepspaces=true]
Bool schema_swap_pre (Pointer i, Pointer j) {
     Return (i>=0 and j>=0 and i<|$\Omega$| and j<|$\Omega$| and i<j);
}

void schema_swap_eff (Pointer i, Pointer j) {
     Variable aux;
     aux:= vector(i);
     vector(i):= vector(j); 
     vector(j):= aux;
}
\end{lstlisting}
\end{small}
\end{tabular}
	\caption{\small Pointer-based representation of the {\em preconditions} and the {\em effects} of the {\tt\small swap} action schema. When applicable, the {\tt\small swap} action schema exchanges the value of the state variables indexed by its two {\em parameters}, the pointers $i$ and $j$.}
	\label{fig:ascheme}
\end{figure}

\subsection{Classical planning with pointers}
Here we extend the {\em classical planning} model (introduced in Section~\ref{sec:background}) with a set of pointers and their corresponding primitive operations. We formalize the set of primitive operations over pointers leveraging the notion of {\em Random-Access Machine}. In more detail, we extend the classical planning model with a RAM machine of $|Z|+2$ registers; $|Z|$ {\em pointers} that reference the planning objects, plus two FLAGS registers (the {\em zero} and the {\em carry}). The two FLAGS registers store the outcome of {\em three-way comparisons}, between two different pointers or between two state variables that are indirectly addressed via function symbols and pointers. 

Given a classical planning instance $P=\tup{X,A,I,G}$, such that the state variables and actions are generated with the set of functions $\Phi$ and action schemes $\Xi$ of a given domain, and the set of objects $\Omega$, then the {\em extended classical planning instance} with a RAM machine of $|Z|+2$ registers is defined as $P_Z'=\tup{X_Z',A_Z',I_Z',G}$, where: 
\begin{itemize}
	\item The new set of {\bf state variables}
	$X_Z'$ comprises:
	\begin{itemize}
		\item The state variables $X$ of the original planning instance, such that each state variable $x_i\in X$ is $x_i\equiv \phi(\overrightarrow{o})$ with $\phi\in\Phi$ and $\overrightarrow{o}\in\Omega^{ar(\phi)}$, as defined above.
		\item Two {\em Boolean variables} $Y=\{y_z,y_c\}$, that play the role of the {\em zero} and  {\em carry} FLAGS registers, respectively.
		\item The {\em pointers} $Z$, a set of extra state variables s.t. each $z\in Z$ has finite domain {\small $D_z=[0,\ldots,|\Omega|-1]$}.
		\item A set of {\em derived state variables} $X_Z = \{ \; \phi(\overrightarrow{z}) \; | \; \phi\in\Phi, \overrightarrow{z} \in Z^{ar(\phi)} \; \}$ whose value is given by the interpretations of the functions of the domain with the corresponding pointers.		
	\end{itemize}
Given a fixed number of pointers, $Y$, $Z$, and $X_Z$, are a subset of state variables that is shared by all the instances that belong to the same domain, no matter the number of objects of each instance.

	\item The new {\bf set of actions} $A_Z'$ will represent the set of actions that is shared for the different classical planning instances in a given domain, and it includes: 
		\begin{itemize}
		\item The {\bf planning actions} $A'$ that result from reformulating each action scheme $\xi\in \Xi$ into its corresponding pointer-based version. The reformulation is a two-step procedure that requires that $Z$ contains, at least, as many pointers as the largest arity of an scheme in $\Xi$: (i), each parameter in $par(\xi)$ is replaced with a pointer in $Z$ and (ii), preconditions and effects are rewritten to refer to these pointers. 
		\item The {\bf RAM actions} that implement the following sets of RAM instructions $\{{\tt\small inc}(z_1)$, ${\tt\small dec}(z_1)$, ${\tt\small cmp}(z_1,z_2)$, ${\tt\small set}(z_1,z_2)$ $| \; z_1,z_2 \in Z\}$ over the pointers in $Z$, and $\{{\tt\small test}(\phi(\overrightarrow{z_1})), {\tt\small cmp}(\phi(\overrightarrow{z_1}),\phi(\overrightarrow{z_2}))\; |\; \overrightarrow{z_1},\overrightarrow{z_2} \in Z^{ar(\phi)} \}$ over the lists of pointers in $Z^{ar(\phi)}$ for each function symbol $\phi\in\Phi$. Respectively, these RAM instructions {\em increment}/{\em decrement} a pointer by one, {\em compare} two pointers, {\em set} the value of a pointer $z_2$ to another pointer $z_1$, {\em test} whether  $\phi(\overrightarrow{z_1})$ is greater than zero, and {\em compare}\footnote{${\tt\small cmp}(\phi(\overrightarrow{z_1}),\phi(\overrightarrow{z_2}))$ instructions are only defined for numeric functions.} the value of $\phi(\overrightarrow{z_1})$ and $\phi(\overrightarrow{z_2})$.
		Each RAM action also updates the $Y=\{y_z,y_c\}$ {\sc flags}, according to the result of the corresponding RAM instruction (which is denoted here by $res$):
\begin{small}
\begin{align*}
 inc(z_1) &\implies res := z_1 + 1,\\
 dec(z_1) &\implies res := z_1 - 1,\\
 cmp(z_1,z_2) &\implies res := z_1 - z_2,\\
 set(z_1,z_2) &\implies res := z_2,\\
 test(\phi(\overrightarrow{z_1})) &\implies res := \phi(\overrightarrow{z_1}),\\ 
 cmp(\phi(\overrightarrow{z_1}),\phi(\overrightarrow{z_2})) &\implies res := \phi(\overrightarrow{z_1}) - \phi(\overrightarrow{z_2}),\\ 
 y_z &:= ( res == 0 ),\\
 y_c &:= ( res > 0 ).
\end{align*}
\end{small}
\end{itemize}	
	\item The new {\bf initial state} $I_Z'$ is the initial state of the original planning instance, but extended with all {\em pointers} set to zero and the two {\em FLAGS} set to {\tt\small False}. The {\bf goals} are the same as those of the original planning instance.
\end{itemize}

{\bf Example}. Here we extend the classical planning instance $P_1=\tup{X,A,I_1,G_1}$ (illustrated in Figure~\ref{fig:cp-example}) with a RAM machine of $|Z|+2$ registers, where $Z=\{i,j\}$ is a set of two pointers. According to this extension, our pointer-based representation of the sequential plan $\pi_1=\tup{swap(o_0,o_5), swap(o_1,o_2), swap(o_1,o_3)}$ is the following sequence of thirteen actions $\pi_1'=\tup{inc(j)^5, swap(i,j), inc(i), dec(j)^3, swap(i,j), inc(j), swap(i,j)}$; where superscripts refer to the number of times that an instruction is sequentially repeated, and where $swap(i,j)$ refers to the pointer-based action schema defined in Figure~\ref{fig:ascheme}. Likewise, our pointer-based version of the sequential plan $\pi_2=\tup{swap(o_0,o_2), swap(o_3,o_5)}$, that solves the classical planning problem $P_2$ illustrated in Figure~\ref{fig:cp-example}, is the ten-action sequence $\pi_2'=\tup{inc(j)^2, swap(i,j), inc(i)^3, inc(j)^3, swap(i,j)}$. Note that any sequential plan for solving a classical planning instance from the vector sorting domain,  no matter the number of state variables and no matter the domain size of these state variables, can be built using exclusively actions from the following set: $\{inc(i), inc(j), dec(i), dec(j), swap(i,j)\}$.

\subsubsection{Theoretical properties}

\begin{theorem}
Given a classical planning instance $P$, its extension $P'_Z$, with a RAM machine of $|Z|$ pointers, preserves the solution space of $P$.
\end{theorem}

\begin{proof}
$\Rightarrow$: Let $\pi=\tup{a_1,\ldots,a_m}$ be a plan that solves $P$, an equivalent plan $\pi'$ that solves $P'_Z$ is built as follows; for each action $a_i\in\pi$, $A'$ contains a pointer-based action schema $a_i'$ that replaces each parameter in $par(a_i)$ with a pointer $z\in Z$. For each such pointer $z$, the plan repeatedly applies RAM actions {\tt\small inc($z$)} or {\tt\small dec($z$)} until they reference the associated vector of objects $\overrightarrow{o}$, and then it applies $a_i'$. The resulting plan $\pi'$ has exactly the same effect as $\pi$ on the original planning state variables in $X$, and since the goal condition of $P'_Z$ is the same as that of $P$, it follows that $\pi'$ solves $P'_Z$.

$\Leftarrow$: Let $\pi'=\tup{a_1',\ldots,a_m'}$ be a plan that solves $P'_Z$. Identify each action in $A'$ among those of $\pi'$, and execute $\pi'$ to identify the assignment of objects to pointers when applying each action in $A'$. Construct a plan $\pi$ corresponding to the subsequence of actions in $A'$ from $\pi'$, replacing each action schema $a_i'\in A'$ by an original action $a_i\in A$ and choosing as parameters of $a_i$ the objects referenced by the pointers of $a_i'$ at the moment of execution. Hence $a_i$ has the same effect as $a_i'$ on the state variables in $X$, implying that $\pi$ has the same effect as $\pi'$ on $X$. Since the goal condition of $P$ is the same as that of $P'_Z$, it follows that $\pi$ solves $P$.
\end{proof}

Our extension of a classical planning problem with a RAM machine of $|Z|+ 2$ registers preserves the solution space of the original problem. Sequential plans in the extended planning model are however longer when pointers must be incremented (or decremented) multiple times to access the corresponding objects, before the corresponding  action schema is executed. For instance, our pointer-based version of plan $\pi_1$ required thirteen steps while the original sequential plan only had three steps. Likewise our pointer-based version of plan $\pi_2$ required ten steps while the original sequential plan only had two steps. On the other hand, our extension does not require explicit action grounding; it is the planner itself who determines the values of the pointers that now feed the action schemas, so in a sense the planner is in charge of doing a partial instantiation. Further, our extension separates the instantiation of each parameter of an action schema. The computation of sequential plans with our pointer-based formalism is however out of the scope of this paper. We will exclusively address the computation of planning solutions represented as planning programs.  

\begin{theorem}
\label{theor:actions}
The new set of actions $A_Z'$ is independent of the number of objects $\Omega$, state variables $X$, and their domain size. 
\end{theorem}

\begin{proof}
The number of actions of a  classical planning instance, extended with a RAM machine of $|Z|$ pointers, is
\begin{equation}
	|A_Z'|=2|Z|^2 + \sum_{\phi\in\Phi} |Z|^{2ar(\phi)} + |A'|.
	\label{eq:nactions}
\end{equation}
This number exclusively depends on the number of pointers in $Z$ and on the arity of the functions in $\Phi$ and the action schemes in $\Xi$.  First, the {\em increment}/{\em decrement} instructions induce $2|Z|$ actions, the {\em set} instructions over pointers induce $|Z|^2-|Z|$ actions, and {\em comparison} instructions of pointers induce $|Z|^2-|Z|$ actions. The comparison instructions can compare two pointers but for symmetry breaking, we only consider the single parameter ordering $(z_i,z_j)$ where $i<j$, i.e. we consider~{\tt\small cmp($z_1$,$z_2$)} but not {\tt\small cmp($z_2$,$z_1$)}. Second, {\em test} instructions are defined over each function symbol and list of pointers with the same size as its arity, inducing $\sum_\phi |Z|^{ar(\phi)}$ actions, and {\em comparison} of predicates with pointers induce $\sum_\phi (|Z|^{2ar(\phi)} - |Z|^{ar(\phi)})$ actions. Therefore, the total number of RAM instructions are $2|Z|+2(|Z|^2-|Z|)+\sum_\phi (|Z|^{ar(\phi)} + |Z|^{2ar(\phi)}-|Z|^{ar(\phi)}) =2|Z|^2 + \sum_\phi|Z|^{2ar(\phi)}$ which only depends on the number of pointers in $Z$ and the arity of each function symbol $\phi$.
Last, as defined by our abstraction procedure, the number of actions in $A'$ is given by the number of parameters of the actions schemes $\Xi$ and the number of pointers in $Z$ to replace these parameters. This means that the size of $A'$ is upper bounded by $|A'|\leq \sum_{\xi\in \Xi} |Z|^{|par(\xi)|}$. As a consequence it follows that $A_Z'$, whose size is given by $|A_Z'| = 2|Z|^2 + \sum_\phi |Z|^{2ar(\phi)} + |A'|$, it is also independent of the number of objects $\Omega$, state variables in $X$ and their domain size.  
\end{proof}

\subsection{\strips\ with pointers}
Since the early 70's, the \strips{} representation formalism is widely used for research in {\em automated planning}~\cite{fikes1971strips}. Even today, \strips{} is an essential fragment of PDDL~\cite{haslum2019introduction}, the input language of the {\em International Planning Competition}, and most planners support the \strips{} representation features.  Here we show that our pointer-based view of planning problems and solutions applies also to object-centered planning formalisms, such as \strips\ planning. In fact, our pointer-based formalism can be understood as an instantiation of {\em F-\strips{}}~\cite{geffner2000functional}, where the single level of indirection of pointers over objects is enough to represent~\strips{} problems with constant memory access. 

\strips\ is an object-centered planning formalism that compactly represents the set of states of a transition system using a finite set of {\em objects}, and a finite set of FOL {\em predicates}, that indicate the properties of the objects and their relations. Likewise, \strips\ compactly represents the space of the possible state transitions using FOL {\em operators}, which are defined as a tuple $op=\langle name(op),$ $args(op),$ $pre(op),$ {\it eff}$\,^-(op),$ {\it eff}$\,^+(op)\rangle$ and where, $name(op)$ is a unique identifier of the operator, $args(op)$ is a set of variable symbols specifying the arguments of the operator, and $pre(op)$, {\it eff}$\,^-(op),$ {\it eff}$\,^+(op)$ are sets of FOL predicates, with variables exclusively taken from $args(op)$, and that respectively specify the {\em preconditions}, {\em negative effects} and {\em positive effects}. The representation of a \strips\ problem is completed specifying an {\em initial} state, that defines the initial situation for all the objects, and the aimed set of {\em goal states}, which is typically specified as a partial state. 

\begin{figure}
\centering
\begin{minipage}{0.4\textwidth}
\begin{tikzpicture}[block/.style= {rectangle, draw=black, thick, text centered, node distance=.5cm}]
\node [block] (b1) [] {1};
\node [block] (b2) [below of= b1] {2};
\node [block] (b3) [below of=b2] {3};
\draw [thick] (-1,-1.25) -- (1,-1.25) node [] {};
\end{tikzpicture}
\end{minipage}
\hspace{-3.cm}
\begin{minipage}[t]{0.5\textwidth}
\begin{small}
\begin{tabular}{c|c|c}
                  & \multicolumn{2}{c}{\bf State representation}\\
  {\bf Predicate} & {\bf Strips} & {\bf Boolean functions}\\\hline
{\tt (clear ?x)} & {\tt (clear b1)} & {\tt clear(b1)}=1 \\
{\tt (handempty)} & {\tt (handempty)} & {\tt handempty()}=1 \\
{\tt (holding ?x)} & - & - \\
{\tt (on ?x ?y)} & {\tt (on b1 b2) (on b2 b3)} & {\tt on(b1,b2)}=1, {\tt on(b2,b3)}=1\\
{\tt (ontable ?x)} & {\tt (ontable b3)} & {\tt ontable(b3)}=1 
\end{tabular}
\end{small}
\end{minipage}
	\caption{\small Example of a three-block  state from the {\em blocksworld}  (left), and its corresponding  representation as a vector of bits (right).} 
	\label{fig:tower}
\end{figure}

{\bf State representation}. In our pointer-based formalism for~\strips{} problems, each state variable $x\in X$ has domain $D_x = \{0,1\}$, and it is built as a FOL \strips\ predicate $\phi\in\Phi$ grounded by a vector of objects $\overrightarrow{o}\in\Omega^{ar(\phi)}$. Figure~\ref{fig:tower} shows the representation of a {\em blocksworld} state using the \strips\ formalism as well as using our formalism. In this state there are three blocks, $\Omega=\{b1, b2, b3\}$, that are stacked in a single tower. Predicates {\small\tt clear(?x)}, {\small\tt holding(?x)}, and {\small\tt ontable(?x)}, are encoded as three different Boolean functions that map each vector of objects to either $0$ or $1$ in the current state. Omitted state variables are assumed to be zero valued. Our vector $X$ of state variables is the result of unifying all the previous predicate and object tuple valuations into a vector. The length of the vector of state variables is then upper bounded by $|X|\leq \sum_{k\geq 0} n_k|\Omega|^k$, where $n_k$ is the number of first-order predicates with arity $k$. For instance, the $X$ vector contains at most $|\Omega|^2 + 3|\Omega| + 1$ state variables for the {\em blocksworld} domain. State-invariants (e.g. in the {\em blocksworld} a block cannot be on top of two different blocks simultaneously) can also be leveraged to save space for the memory allocation of the state variables.

{\bf Action representation}. 
Given a FOL \strips\ operator $op=\langle name(op),$ $args(op),$ $pre(op),$ {\it eff}$\,^-(op),$ {\it eff}$\,^+(op)\rangle$, our pointer-based formalism produces its corresponding pointer-based action schema $\langle${\it name, params, pre, eff}$\rangle$:
\begin{itemize}
\item The name of the action schema is $name(op)$, the name of the given FOL \strips\ operator.
\item For each argument in $args(op)$, the action schema has a {\em pointer} that indexes an object $o \in \Omega$.
\item The set $pre(op)$ is transformed into a conjunctive arithmetic-logic expression with conditions of two kinds: (i) for each {\em pointer} in the parameters of the action schema, the  conditions asserting that the pointer is within its domain and (ii), for each precondition in $pre(op)$ a condition  asserting that the state variable addressed by the pointers content equals to some specific value of its domain. 
\item Each negative effect in {\it eff}$\,^-(op)$ is transformed into an indirect variable assignment that sets the corresponding state variable to 0. Likewise, each positive effect in {\it eff}$\,^+(op)$ is transformed into an indirect variable assignment that sets the corresponding state variable to 1.
\end{itemize}

Figure~\ref{fig:unstack2} shows our pointer-based definition for the {\tt\small unstack} action schema from the {\em blocksworld} that implements the corresponding operator represented in the \strips\ fragment of PDDL of Figure~\ref{fig:unstack1}. The action schema of Figure~\ref{fig:unstack2} is implemented using two {\em pointers} ($i$ and $j$), and it applies to any {\em blocksword} instance, no matter the number of blocks or their identity. In this implementation the state variables are {\em global}, meaning that they can be accessed from any of the action schemas. Again the execution of these procedures is a deterministic matching-free process since the input pointers do always have a block assigned.

\begin{figure}
\begin{small}
\begin{lstlisting}[basicstyle=\ttfamily,columns=fullflexible,keepspaces=true]
(:action unstack
	 :parameters (?x ?y)
	 :precondition (and (clear ?x) (handempty) (on ?x ?y))
	 :effect (and (holding ?x) (clear ?y)
              (not (clear ?x)) (not (handempty)) (not (on ?x ?y)))))
\end{lstlisting}
	\caption{\small The {\tt unstack} \strips\ operator from the {\em blocksworld} domain represented in the PDDL language.}
	\label{fig:unstack1}	
\end{small}
\end{figure}
\begin{figure}
\begin{small}
\begin{lstlisting}[basicstyle=\ttfamily,columns=fullflexible,keepspaces=true]
Bool schema_unstack_pre (Pointer i, Pointer j) {	 
     Return (i>=0 and j>=0 and i<|$\Omega$| and j<|$\Omega$| and
             clear(i)=1 and handempty()=1 and on(i,j)=1);
}	

void schema_unstack_eff (Pointer i, Pointer j) {
     clear(i) := 0;  handempty() := 0;  on(i,j) := 0;	 
     holding(i) := 1; clear(j) := 1;   
}		   
\end{lstlisting}
\end{small}
	\caption{\small The {\tt\small unstack} action schema from {\em blocksworld} defined with two pointers ($i$ and $j$).}
	\label{fig:unstack2}
\end{figure}

{\bf Problem representation}. We complete our pointer-based representation of a \strips{} problem with the {\em init} and {\em goal} procedures. Figure~\ref{fig:problem} shows the {\em init} and {\em goal} procedures for the planning problem of unstacking the 3-block tower of Figure~\ref{fig:tower}. They have the same formal structure as the procedures we use for our pointer-based representation of the {\em preconditions} and {\em effects} of an action schema, but with no arguments. In more detail: 
\begin{itemize}
    \item The {\em init procedure} is a write-only procedure, that implements a total variable assignment of the state variables for specifying the initial state of the \strips{} problem. 
    \item The {\em goal procedure} is a read-only Boolean procedure, that encodes the state-constraint that specifies the subset of goal states. 
\end{itemize}

{\bf Plan representation}. Following our pointer-based representation, a sequential plan $\pi=\tup{a_1, \ldots, a_n}$ is a sequence of transformations of the state variables using: (i) our pointer-based action schemas that encode the FOL \strips{} operators and (ii), the subset of RAM actions $\{{\tt\small inc}(z)$, ${\tt\small dec}(z) | \; z \in Z\}$ for achieving the aimed binding for each parameter of an action schema.  Our pointer-based representation of the four-action  plan $\pi=\langle${\tt\small unstack(b1,b2)}, {\tt\small putdown(b1)}, {\tt\small unstack(b2,b3)}, {\tt\small putdown(b2)}$\rangle$ for unstacking the three-block tower of Figure~\ref{fig:tower}, is the following sequence of actions $\pi'=\langle inc(j)$, $unstack(i,j)$, $putdown(i)$, $inc(i)$, $inc(j)$, $unstack(i,j)$, $putdown(i)\rangle$ that leverages two {\em pointers} $Z=\{i,j\}$, that are defined over the set of blocks $\Omega$.

In spite of its popularity, the \strips{} representation is too low-level for many interesting   applications~\cite{geffner2003pddl,smith2008anml,barreiro2012europa}. Our pointer-based representation naturally extends beyond \strips\ to more expressive object-oriented  representations. For instance, state variables can also comprise numeric variables (e.g. integers or reals) to implement {\em numeric fluents} as in PDDL2.1~\cite{fox2003pddl2} or in RDDL~\cite{sanner2010relational}. Object {\em typing} can also be  implemented in a straightforward way (e.g.~specializing pointers to the number of objects of a particular type) to compact the size of the vector of state-variables and to optimize the implementation of quantified preconditions/effects/goals~\cite{pednault1994adl}. Last, our pointer-based representation also supports {\em conditional effects}~\cite{nebel2000compilability}, e.g. an action schema can specify multiple variable assignments conditioned by the different values of the state variables.

\begin{figure}
\begin{small}
\begin{lstlisting}[basicstyle=\ttfamily,columns=fullflexible,keepspaces=true]
void init() {
  i := 0; j := 0;
  $y_z$ := 0; $y_c$ := 0;
  clear(b1) := 1;  on(b1,b2) := 1; on(b2,b3) := 1;  ontable(b3) := 1;  
}

Bool goals() {
  Return (ontable(b1)=1 and ontable(b2)=1 and ontable(b3)=1);
}
\end{lstlisting}
\end{small}
	\caption{\small The {\em init} and {\em goal} procedures for representing the \strips{} planning problem of unstacking the three-block tower of Figure~\ref{fig:tower}.}
	\label{fig:problem}
\end{figure}

\section{Generalized planning as heuristic search} 
\label{sec:hsearch}
First the section shows how we build, in an agnostic manner, a GP problem from a set of classical planning instances of a given domain. Then, the section describes in detail our {\em GP as heuristic search} approach: our search space for GP, the evaluation/heuristic functions that we use for guiding the search, and the particular details of our search algorithm.

\subsection{From a set of classical planning instances to a GP problem}
Generalized plans leverage relevant subsets of shared state variables ({\em features}) and actions whose execution is well-defined for any possible value of the state variables of the classical planning instances to be solved. Building a GP problem from a set of classical planning instances of a given domain is not trivial because these two ingredients may not be given in the representation of the classical planning instances.

On the one hand given a classical planning domain, the specification of a set of features that is (i), {\em expressive} enough to represent a polynomial solution valid for any instance in the domain and (ii), {\em compact} enough for the effective computation of that solution, is a complex task that requires expert knowledge on both the domain and the aimed solution. In fact, the automatic specification of {\em expressive} and {\em compact} features for a planning domain is a challenging research question that is investigated since the early days of  automated planning~\cite{veloso1995integrating}. On the other hand the set of ground actions for the different instances of a given domain, is usually different since it depends on the number of objects. Back to the sorting example illustrated in Figures~\ref{fig:cp-example} and~\ref{fig:sortnet}, the classical planning instances for sorting a vector of length six induced $\frac{6\times 5}{2}$ $swap(o_i,o_j)$, $i<j$ actions, while instances for sorting a vector of length seven would induce a set of  $\frac{7\times 6}{2}$ $swap(o_i,o_j)$ actions.

Given a finite and non-empty set of $T$ classical planning instances from a given domain, our approach for automatically building a GP problem is to extend the instances with a RAM machine of $|Z|$ pointers. The result is a GP problem $\mathcal{P}=\{P_1,\ldots,P_T\}$, where each instance $P_t\in\mathcal{P}$,  {\small $1\leq t \leq T$} may differ in the actual set of objects, initial state, and goals, but all instances necessarily share the subset of state variables $X_Z\cap Y\cap Z$ and the same set of actions $A_Z'$. Formally, $P_1=\tup{X'_{Z_1},A_Z',I'_{Z_1},G_1}, \ldots, P_T=\tup{X'_{Z_T},A_Z',I'_{Z_T},G_T}$ where $\forall_{1\leq t\leq T} X_{Z_t}\subset X'_{Z_t}$, $\forall_{1\leq t\leq T} Y\subset X'_{Z_t}$, $\forall_{1\leq t\leq T} Z\subset X'_{Z_t}$, and $\forall_{a'\in A_Z'} par(a')\in Z^{ar(a')}$. The number of pointers $|Z|$ is a {\em parameter} that indicates how many pointers are used in the extension of the classical planning instances \footnote{At least $Z$ must contain as many pointers as the largest arity of the functions $\Phi$ and action schemes $\Xi$ of the given domain.}. 

Our extension with a RAM machine of $|Z|$ pointers automatically defines a minimalist but general set of features for the set of a classical planning instances from a given domain. 
\begin{definition}[The feature language]
	\label{def:flanguage}
We define the feature language as the four possible joint values of the two Boolean variables $Y=\{y_z,y_c\}$, and we denote this language as $\mathcal{L}=\{(\neg y_z\wedge \neg y_c), (y_z\wedge \neg y_c), (\neg y_z\wedge y_c), (y_z\wedge y_c)\}$.  
\end{definition}

We say $\mathcal{L}$ is minimalist because it only contains four elements, and we say $\mathcal{L}$ is general because it is independent of the number of objects and hence, of the domain of the state variables and the number of state variables. Note that our features are a function of (i) the state variables and (ii) the last executed action, since they all may affect the value of $Y=\{y_z,y_c\}$. Such notion of feature is related to the notion of state {\em observation} in the POMDP formalism, where observations depend on the current state and the action just taken~\cite{roy2005finding}. With this regard it can be understood that our GP approach computes, at the same time, a generalized plan and an {\em observation function} useful for that generalized plan. Our feature language is also related to {\em Qualitative numeric planning}~\cite{srivastava2011qualitative,bonet2019qualitative,illanes2019generalized} which leverages propositions to abstract the value of numeric state variables. Given that our FLAGS $Y=\{y_z,y_c\}$ depend on the last executed action, and considering that only RAM instructions update the variables in $Y$, we have an observation space of $2^{|Y|}\times (2|Z|^2 + \sum_\phi |Z|^{2ar(\phi)})$ state observations implemented with only $|Y|$ Boolean variables. The four joint values of $\{y_z,y_c\}$ model then a large space of observations, e.g.~$=\,$0, $\neq\,$0, $<0, >0, \leq 0, \geq 0$ as well as relations $=, \neq, <, >, \leq, \geq$ on pairs of state variables.

Likewise, our extension with a RAM machine of $|Z|$ pointers automatically defines the shared set of actions $A_Z'$, that is well-defined for the set of a classical planning instances from a given domain. Because the set of pointers $Z$ is fixed for  the $T$ input classical planning instances we have that, after our extension, all the instances share the same set of actions $A_Z'$. The execution of the actions in $A_Z'$ is well-defined over the subset of state variables $Z$, no matter the actual number of objects, or the corresponding number and domain size of the state variables; we recall the reader that the set of actions $A_Z'$ exclusively depends on the number of pointers $|Z|$ and the arity of actions and functions (Theorem~\ref{theor:actions}).

\subsection{The search space}
Briefly, our {\em GP as heuristic search} approach implements a combinatorial search in the solution space of the possible planning programs. Next we provide more details on how we implemented a tractable search space for GP.

\begin{definition}[Partially specified planning program]
	\label{def:partialp}
A {\em partially specified planning program} is a planning program such that the content of some of its program lines may be undefined.
\end{definition}

Each node of our search space is a {\em partially specified planning program} which is binary encoded as follows. Given a set of state variables $X$, a set of actions $A$, a maximum number of program lines $n$ such that the last instruction is $w_{n-1}=end$, and defining the propositions of $\mathsf{goto}$ instructions as $(x=v)$ atoms where $x\in X$ and $v\in D_x$, we have that the space of possible planning programs is represented by the following bit-vectors:
\begin{enumerate}
	\item The {\em action vector} of length $(n-1)\times |A|$, indicating whether an action $a\in A$ is programmed on line $0\leq i< n-1$.
	\item The {\em transition vector} of length $(n-1)\times (n-2)$, indicating whether a $\mathsf{go}(i',*)$ instruction is programmed on line $0\leq i< n-1$. 
	\item The {\em proposition vector} of length $(n-1)\times\sum_{x\in X}|D_x|$, indicating whether a $\mathsf{go}(*,!\tup{x=v})$ instruction is programmed on line {\small $0\leq i< n-1$}.
\end{enumerate} 
A {\em partially specified planning program} is then encoded as the concatenation of these three bit-vectors and the length of the resulting bit-vector is: 
\begin{equation}
	(n-1) \left( |A|+ (n-2) + \sum_{x\in X}|D_x| \right).
	\label{eq:size}
\end{equation}

The binary encoding allows us to quantify the similarity of two {\em partially specified planning programs} (e.g. the {\em Hamming distance} of their corresponding bit-vector representation) and more importantly, to systematically enumerate the space of all the possible planning programs with a maximum of $n$ lines. Let us define the {\em empty program} as the particular partially specified planning program whose  instructions are all undefined (i.e. all bits of its bit-vector representation are set to {\tt\small False}). Starting from the {\em empty program}, we can enumerate the entire set of possible planning programs with two search operators: 
\begin{itemize}
    \item {\tt program(i,a)}, that programs an action $a\in A$ at line $i$ of a program 
    \item {\tt program(i,i',x,v)}, that programs a $\mathsf{goto}(i',!\tup{x=v})$ instructions at line $i$ of a program. 
\end{itemize}

These two search operators are only applicable when $i$ is an undefined program line (i.e. in the bit-vector representation the bits corresponding to the encoding of the program line $i$ are set to {\tt False}). Given the bit-vector representation of a {\em partially specified planning program}, the application of the {\tt program(i,a)} or {\tt program(i,i',x,v)} search operators set to {\tt True} the corresponding bits. With this regard, the partially specified planning program of a given search node is at {\em Hamming distance 1} from its parent, when programming a planning action with {\tt program(i,a)}, or at {\em Hamming distance 2}, when programming a $\mathsf{goto}$ instruction with {\tt program(i,i',x,v)}. In fact, this is the search space leveraged by the classical planning compilation approach for computing planning programs with an off-the-shelf classical planner~\cite{segovia2019computing}. Equation~\ref{eq:size} reveals that the number of planning programs with $n$ lines depends on the number of grounded actions $|A|$, the number of state variables $x\in X$, and their domain $D_x$. This dependence causes an important scalability issue, limiting the applicability of the cited compilation to planning instances of contained size. 

\begin{definition}[The GP search space]
	\label{def:gp-space}
	Given a {\em GP problem} $\mathcal{P}$, that is built extending a set $\{P_1,\ldots,P_T\}$ of classical planning instances from a given domain with a RAM machine of $|Z|$ pointers. Our GP search space is the set of partially specified planning programs that can be built with $n$ program lines, the set of planning actions $A_Z'$, and the set of $\mathsf{goto}$ instructions that are exclusively conditioned on a feature in $\mathcal{L}$.
\end{definition}

Definition~\ref{def:gp-space} leverages our minimalist feature language $\mathcal{L}$ to build a tractable solution space for GP. We represent GP solutions as {\em planning programs} where $\mathsf{goto}$ instructions can exclusively be conditioned on a feature in $\mathcal{L}$.  Limiting the conditions of $\mathsf{goto}$ instructions to any of the four features in $\mathcal{L}$ greatly reduces the number of planning programs with $n$ lines, specially when state variables have large domains (i.e.~{\em integers}); the {\em proposition vector} required to encode a planning program becomes now a vector of only $(n-1)\times 4$ bits (one bit for each of the four features in $\mathcal{L}$). Equation~\ref{eq:size} simplifies then to:
\begin{equation}
	(n-1) \left( |A_Z'| + (n-2) + 4 \right).
	\label{eq:size2}
\end{equation}
Equation~\ref{eq:size2} shows that the size of our new solution space for GP is independent of the number of objects and hence the number of original state variables and their domain size; Theorem~\ref{theor:actions} already showed that $A_Z'$ no longer grows with the number of objects. This novel GP solution space can now scale to planning problems where state variables have large domains (e.g. integers) and that have a large number of state variables. 

Last but not least, since the $Y=\{y_z,y_c\}$ {\sc FLAGS} store the outcome of three-way comparisons, the fourth case $(y_z\wedge y_c)\in \mathcal{L}$ can actually never happen as a result of a comparison. This fourth case is however useful for representing unconditional {\tt\small gotos}.

{\bf Example.} Figure~\ref{fig:lists} shows two examples of planning programs that can be found by our {\sc BFGP} algorithm, that searches in our solution search-space for GP: (left) a generalized plan for reversing a list, and (right) a generalized plan for sorting a list. Note that in both planning programs $\mathsf{goto}$ instructions are exclusively conditioned on a feature in $\mathcal{L}$, and that both planning programs are actually solutions for an infinite set of classical planning problems; they generalize with a $\mathsf{swap}$ action schema of arity $2$ and a $vector$ predicate symbol of arity $1$, no matter the number of objects $\Omega$ and no matter the state variables content, i.e. $x_i\equiv vector(o_i)$ such that $o_i\in \Omega$, $x_i\in X$ and $D_{x_i} \subseteq \mathbb{N}_0$.  In the planning program for reversing a list (left), line 0 sets the pointer $j$ to the last element of the list. Then, line 1 swaps in the $vector$ the element pointed by $i$ (initially set to zero) and the element pointed by $j$, pointer $j$ is decremented, pointer $i$ is incremented, and this sequence of instructions is repeated until the condition on line 5 becomes false, i.e when $j>i$, which means that reversing the list is finished.  The planning program for {\em sorting} a list (right) is actually an implementation of the {\em selection-sort} algorithm. In this program, pointers $j$ and $i$ are used for inner (lines 5-7) and outer (lines 8-11)  loops respectively, and $min$ to point to the minimum value in the inner loop (lines 3-4); $\neg y_z\wedge \neg y_c$ on line $2$ represents whether the content of $vector(j)$ is less than the content of $vector(min)$, while $y_z\wedge\neg y_c$ on line 7 represents whether $j==length$ (resp.~$i==length$ on line 11).

\begin{figure}
\small
		\begin{subfigure}[t]{.45\columnwidth}
			\begin{lstlisting}[mathescape]
            0. set(j,tail)
            1. swap(i,j)
            2. dec(j)
            3. inc(i)
            4. cmp(j,i)
            5. goto(1, $\neg (\neg y_z\wedge\neg y_c)$)
            6. end			
			\end{lstlisting}
		\end{subfigure}	
\hspace{-2.5cm}		
		\begin{subfigure}[t]{.45\columnwidth}
			\begin{lstlisting}[mathescape]
			0. set(min,i)
			1. cmp(vector(j),vector(min))
			2. goto(5, $\neg (\neg y_z\wedge \neg y_c)$)
			3. set(min, j)
			4. swap(i,min)
			5. inc(j)
			6. cmp(length,j)
			7. goto(1, $\neg (y_z\wedge\neg y_c)$)
			8. inc(i) 
			9. set(j,i)
			10. cmp(length,i)
			11. goto(0, $\neg (y_z\wedge\neg y_c)$)    
			12. end
			\end{lstlisting}
		\end{subfigure}	
	\caption{\small Two examples of generalized plans: (left) for {\em reversing} a list; (right) for {\em sorting} a list with the {\em selection-sort} algorithm.}
	\label{fig:lists}
\end{figure}

\subsubsection{Theoretical properties} 

\begin{theorem}
The space of planning programs that 
exclusively branch over the features in $\mathcal{L}$ preserves the solution space of planning programs.
\end{theorem}

\begin{proof}
$\Rightarrow$: Given a GP problem $\mathcal{P}$ and a planning program $\Pi$, that solves $\mathcal{P}$ and whose {\tt\small goto} instructions are exclusively conditioned on the features in $\mathcal{L}$. An equivalent planning program whose execution flow branches with the original $\mathsf{goto(i,!(x=v))}$ instructions is built by defining a set of four Boolean state variables; one Boolean variable for each feature in $\mathcal{L}$, that are mutually exclusive, and whose value is actually given by the joint value of the $Y=\{y_z,y_c\}$ variables. With these new Boolean state variables defined, we can replace each feature in the condition of a $\mathsf{goto}$ instruction by a $!(x=1)$ condition checking whether its corresponding Boolean variable equals 1.

$\Leftarrow$: Given a GP problem $\mathcal{P}$ and a planning program $\Pi$ that solves $\mathcal{P}$. An equivalent planning program, that exclusively branches over any of the features in $\mathcal{L}$, is built replacing each $\mathsf{goto(i,!(x=v))}$ instruction in $\Pi$, where $x \equiv \phi(\overrightarrow{o})$ s.t. $\phi\in\Phi$ and $\overrightarrow{o}\in\Omega^{ar(\phi)}$, by a finite block of instructions that: (i) increments/decrements a vector of auxiliary pointers $\overrightarrow{z_{aux}}$, with size $ar(\phi)$, until they indirectly address objects $\overrightarrow{o}$, (ii) given auxiliary static state variables for each possible value, i.e. $\forall_{v\in D_x} x_v$, and a dedicated object for each new state variable $o_v$ such that $x_v \equiv \phi(o_v)$, increments/decrements another auxiliary pointer $z_{static}$ in a function $\phi{static}(z_{static})$ until it reaches object $o_v$ such that $x_{v}\equiv \phi{static}(o_v)$ which equals $v$, (iii) compares the content of these two state variables with a $\mathsf{cmp}(\phi(\overrightarrow{z_{aux}}),\phi{static}(z_{static}))$ instruction and (iv), jumps to the $i$ target line when the state variables differ in their content with a $\mathsf{goto(i,!(y_z\wedge\neg y_c))}$ instruction.
\end{proof}

Note however that our GP solution space may still be incomplete in the sense that either the bound $n$ on the maximum number of program lines, or the maximum number of pointers available $|Z|$, may be too small to accommodate a solution to a given GP problem.  In that case, a higher-level combinatorial search can be implemented to incrementally find a suitable number of required program lines and pointers. For instance, like is done in SAT-planning where the planning horizon is iterative incremented until it is large enough to accommodate a solution plan~\cite{rintanen2012planning}.

\subsection{The evaluation functions}
Here we define two different families of evaluation functions, that exploit two different sources of information, to guide a combinatorial search in our GP solution space of partially specified planning programs:
\begin{itemize}
	\item {\em The program structure}. Given a partially specified planning program $\Pi$, we define a set of evaluation functions $f(\Pi)$, that establish different kinds of {\em preferences/priors} on the structure of the aimed generalized plans. For instance, following the {\em Occam's razor principle} a structural function can prefer generalized plans of simpler complexity or it can prefer generalized plans with more programmed lines so program execution failures can be detected earlier in the search.
		\begin{itemize}
			\item$f_1(\Pi)$, the number of $\mathsf{goto}$ instructions in $\Pi$.
			\item$f_2(\Pi)$, the number of {\em undefined} program lines in $\Pi$.
			\item$f_3(\Pi)$, the number of repeated actions in $\Pi$,
			\item$f_7(\Pi)$, the max number of {\em nested} $\mathsf{goto}$ instructions in $\Pi$. A $\mathsf{goto}$ instruction jumps from an origin program line to a destination program line. We say that a $\mathsf{goto}$ instructions is {\em nested} when it appears within the origin and destination lines of another $\mathsf{goto}$ instruction.
		\end{itemize}
	
	\item {\em The empirical performance of the program}. Given a partially specified planning program $\Pi$ and a GP problem $\mathcal{P}=\{P_1,\ldots,P_T\}$, we define a set of evaluation functions $f(\Pi,\mathcal{P})$ that assess the performance of $\Pi$ on $\mathcal{P}$,  executing $\Pi$ on each of the classical planning instances $P_t\in\mathcal{P}$, {\small $1\leq t\leq T$}. Section~\ref{sec:background} defined the execution of a planning program on a classical planning instance as a deterministic procedure that terminates either succeeding to solve that instance or failing it. Likewise the execution of a partially specified planning program is a deterministic procedure that introduces a new termination case, {\em reaching an unspecified program line}. When the program execution terminates because an unspecified program line is reached, $f(\Pi,\mathcal{P})$ functions can be used to assess the cost of that program execution, as well as to estimate how far is the program from solving the given GP problem, which define {\it cost} and {\it heuristic} functions for GP respectively.
		\begin{itemize}
			\item$f_4(\Pi,\mathcal{P})=n-max_{P_t\in\mathcal{P}} f_4(\Pi,P_t)$, where $f_4(\Pi,P_t)$ returns the number of the undefined program line eventually reached after executing $\Pi$ on the classical planning instance $P_t\in \mathcal{P}$. 
			\item $f_5(\Pi,\mathcal{P})=\sum_{P_t\in \mathcal{P}} f_5(\Pi,P_t)$, where \[f_5(\Pi,P_t)=\sum_{x\in X_t} (v_x-G_t(x))^2.\] Here, $v_x\in D_x$ is the value eventually reached, for the state variable $x\in X_t$, after executing $\Pi$ on the classical planning instance $P_t\in\mathcal{P}$, and $G_t(x)$ is the value for this same variable as specified in the goals of $P_t$. Note that for Boolean variables the squared difference becomes a simple difference. This means that for \strips\ planning problems, where all the state variables are Boolean, $f_5(\Pi,P_t)$ is actually a counter of how many atomic goals in $G_t$ are still not true.
			\item $f_6(\Pi,\mathcal{P}) = \sum_{P_t\in{\cal P}} |exec(\Pi,P_t)|$, where $exec(\Pi,P_t)$ is the sequence of actions induced from executing the planning program $\Pi$ on the planning instance $P_t$.
			\item $f_8(\Pi,\mathcal{P}) = f_5(\Pi,\mathcal{P}) + f_6(\Pi,\mathcal{P})$ is the sum of an estimation to the goal and the total accumulated cost, akin to an evaluation function for $A^*$ searching algorithm.
			\item $f_9(\Pi,\mathcal{P}) = W\cdot f_5(\Pi,\mathcal{P}) + f_6(\Pi,\mathcal{P})$ is similar to $f_8$ but the estimation to the goal is multiplied by a factor $W$, which is set to $5$ by default, akin to an evaluation function for $WA^*$ searching algorithm.
		\end{itemize}
\end{itemize}

All these functions are {\em evaluation functions} (i.e.~smaller values are preferred). The structural functions $f_1(\Pi)$, $f_2(\Pi)$, $f_3(\Pi)$ and $f_7(\Pi)$, are all computed in linear time by traversing the bit-vector representation of $\Pi$. On the other hand, the computation complexity of the three empirical functions $f_4(\Pi,\mathcal{P})$, $f_5(\Pi,\mathcal{P})$, $f_6(\Pi,\mathcal{P})$, $f_8(\Pi,\mathcal{P})$ and $f_9(\Pi,\mathcal{P})$ is given by the complexity of the partially specified program $\Pi$. {\em Performance-based} functions accumulate a set of $T$ costs (one for each classical planning instance in the GP problem) that could actually be expressed as a combination of different aggregation functions, e.g.~{\em sum}, {\em max}, average, weighted average, etc. Functions $f_4(\Pi,\mathcal{P})$ and $f_5(\Pi,\mathcal{P})$ are the only {\em cost-to-go heuristic} functions; they provide an estimate on how far is a partially specified planning program from solving a GP problem. With this regard, $f_5(\Pi,\mathcal{P})$ requires that the goal condition of the classical planning instances in a GP problem is specified as a partial state. On the other hand $f_4(\Pi,\mathcal{P})$ does not post any particular requirement on the structure the goal condition, so they can even be a {\em black-box} Boolean procedure over the state variables.

{\bf Example}. We illustrate how our evaluation functions work on the following partially specified program $\Pi= \;${\tt\small 0.swap(i,j)} {\tt\small 1.inc(i)} {\tt\small 2.dec(j)} {\tt\small 3.goto(2,!($y_z \wedge \neg y_c$))} {\tt\small 4\ldots } {\tt\small 5.end}, where only line 4 is not programmed yet. The value of the evaluation functions for this partially specified program is $f_1(\Pi)=1$, $f_2(\Pi)=5-4=1$, $f_3(\Pi)=0$, $f_7(\Pi)=1$. Given the GP problem $\mathcal{P}=\{P_1,P_2\}$ that comprises the two classical planning instances illustrated in Figure~\ref{fig:cp-example}, and pointers $i$ and $j$ starting at the first and last object indexes, respectively, we can compute $f_4$ and $f_5$ to evaluate how far $\Pi$ is from solving the GP problem of sorting lists, the accumulated cost $f_6$, and evaluation functions $f_8$ and $f_9$ that combine heuristic-like functions with accumulated cost. In this case $f_4(\Pi,\mathcal{P})=5-4=1$, $f_5(\Pi,\mathcal{P})=32$, $f_6(\Pi,\mathcal{P})=14+14=28$, $f_8(\Pi,\mathcal{P})=32+28=60$ and $f_9(\Pi,\mathcal{P}) = 32 + 5\cdot 28 = 172$.

\subsection{The search algorithm}
Here we describe our heuristic search algorithm for generalized planning. This algorithm implements a {\em Best-First Search} (BFS) in our GP solution space of the possible partially specified planning programs with $n$ program lines, and a RAM machine with $|Z|$ pointers. Algorithm~\ref{alg:BFGP} shows the pseudo-code of our Best-First Search for Generalized Planning ({\sc BFGP}). 

\begin{algorithm}
\SetAlgoLined
\KwData{A generalized planning problem $\mathcal{P}$, a number of pointers $|Z|$, a number of program lines $n$}
\KwResult{A generalized plan $\Pi$ that solves $\mathcal{P}$}
 Open $\leftarrow \{\Pi_{empty}\}$ \;
 \While{$Open\neq\emptyset$}{
  $\Pi\leftarrow$ extractBestProgram(Open) \;
  ChildrenPrograms $\leftarrow$ expandProgram($\Pi,|Z|$, n) \;
  \For{$\Pi'\in ChildrenProgams$}{
  evaluateProgram($\Pi',\mathcal{P}$)\;
  \uIf{isGoal($\Pi',\mathcal{P}$)}{
  return($\Pi'$)\;
   }
   \uIf{not isDeadEnd($\Pi',\mathcal{P}$)}{
   Open $\leftarrow$ insertProgram(Open,$\Pi'$)\;
   }
  }
  }
 \caption{Best-First Generalized Planning ({\sc BFGP})}
 \label{alg:BFGP}
\end{algorithm}

{\sc BFGP} is a {\em frontier search} algorithm meaning that, to reduce memory requirements, {\sc BFGP} stores only the {\em open list} of generated nodes but not the {\em closed list} of expanded nodes~\cite{korf2005frontier}.  Initially the open list of the {\sc BFGP} algorithm only contains the search node that corresponds to the {\em empty program} of $n$ lines, which is denoted as $\Pi_{empty}$ in Algorithm~\ref{alg:BFGP}. The {\em empty program} is then the root node of the search-tree developed by {\sc BFGP}. The node extraction and node insertion procedures of the {\sc BFGP} algorithm are implemented as in a regular BFS search. Next we provide more details on the particular {\em expansion} and {\em evaluation} mechanisms that are implemented by our {\sc BFGP} algorithm. The {\sc BFGP} algorithm sequentially expands the best node in the {\em open list}, that is implemented as a priority queue, and that is sorted according to the value of our  evaluation/heuristic functions explained above. 

Let $\Pi$ be the partially specified program that corresponds to the best node extracted by {\sc BFGP} from the open list. {\sc BFGP} expands that node generating one successor node for each partially specified program that result from programming the maximum undefined program line that is reached after executing $\Pi$ on all the instances in $\mathcal{P}$. In other words given a partially specified program $\Pi$, only its $max_{P_t\in\mathcal(P)} f_4(\Pi,P_t)$ line is programmable. {\sc BFGP} implements this particular node expansion procedure because it guarantees that duplicate successors are not generated in the {\sc BFGP} search-tree. In addition, this node expansion procedure induces a tractable branching factor of $(|A_Z'| + (n-2)\times 4)$; at a given program line {\sc BFGP} can only program a planning action in $A_Z'$ or a $\mathsf{goto}$ instruction that can jump to $n-2$ different destination program lines, and that is conditioned by any of the four different features in $\mathcal{L}$. The depth of the search tree developed by the {\sc BFGP} algorithm is the number of program lines $n$, since only an undefined line can be programmed. 

Before a new search node is inserted into the open list, the corresponding partially specified program $\Pi'$ is executed on all the classical planning instances in $\mathcal{P}$. This execution is implemented by the node evaluation procedure of the {\sc BFGP} algorithm, and it can result in the three following different outcomes:
\begin{enumerate}
    \item {\em $\Pi'$ is a solution for $\mathcal{P}$}. If the execution of $\Pi'$ solves all the instances $P_t\in \mathcal{P}$, then search ends, and $\Pi'$ will be returned as a valid solution for the GP problem $\mathcal{P}$.
    \item {\em $\Pi'$ fails to solve $\mathcal{P}$}. If the execution of $\Pi'$ on a given instance $P_t\in \mathcal{P}$ fails, this means that the search node corresponding to the partially planning program $\Pi'$ is a dead-end. The search node will be discarded, so $\Pi'$ is not inserted into the open list.
    \item {\em $\Pi'$ may still be a solution for $\mathcal{P}$}. This means that the execution of $\Pi'$ on some of the classical planning instances in $\mathcal{P}$ reached an undefined program line ($\Pi'$ might solve some of the instances in $\mathcal{P}$). As a consequence $\Pi'$ is inserted into the open list, at its corresponding position according to the value of our evaluation/heuristic functions explained above.
\end{enumerate}

{\bf Example}. Let us recover from the previous example the GP problem $\mathcal{P}=\{P_1,P_2\}$, and the 
partially specified program $\Pi= \;${\tt\small 0.swap(i,j)} {\tt\small 1.inc(i)} {\tt\small 2.dec(j)} {\tt\small 3.goto(2,!($y_z \wedge \neg y_c$))} {\tt\small 4\ldots } {\tt\small 5.end}, where lines $[0,3]$ are programmed and only line $4$ is unspecified. Imagine now that {\sc BFGP} extracts this program from the open list because it has the best evaluation value. In this case, the previous execution of $\Pi$ on the classical planning instances $P_1$ and $P_2$, implemented by the node evaluation procedure, ended in both instances at the undefined program line 4. This means that the only programmable line is $4.$ Assuming that two pointers are available (i.e. $Z=\{i, j\}$) we can program any of following twelve actions in line 4. $\{inc(i)$, $inc(j)$, $dec(i)$, $dec(j)$, $cmp(i,j)$, $set(i,j)$, $set(j,i)$, $test(vector(i))$, $test(vector(j))$, $cmp(vector(i),vector(j))$, $swap(i,j)$, $swap(j,i)$ $\}$. A $\mathsf{goto}$ can only be programmed after a RAM action, which is not the case of line 4, since line 3 has another $\mathsf{goto}$ instruction. In other words the search node corresponding to the partially specified program from the previous example would have a total of twelve children nodes that could be added to the open list. In the hypothetical case that previous line 3. would contain a RAM action, a $\mathsf{goto}$ instruction for jumping to lines $[0,3]$ conditioned by the corresponding four features in $\mathcal{L}$ could also be programmed at line 4.

\subsection{Theoretical properties}

\begin{theorem}[Termination]
Given a generalized planning problem $\mathcal{P}$, a finite set of pointers $Z$, and a finite number of program lines $n$, the execution of the {\sc BFGP} algorithm always terminates.
\end{theorem}

\begin{proof}
By definition of the expansion procedure of the {\sc BFGP} algorithm (i), only unspecified lines can be programmed and (ii), any children program always has one more line programmed than its parent. This means that {\sc BFGP} increases the number of programmed lines, until all lines are programmed. When all lines are programmed {\sc BFGP} necessarily terminates, either by succeeding to solve $\mathcal{P}$, or by failing to solve some of the classical planning instances in $\mathcal{P}$. The only possible cause for the non-termination of the BFGP algorithm would be that BFGP could generate duplicate search nodes, allowing  the infinite re-opening of an already discarded node. Again by definition of the expansion procedure of the {\sc BFGP} algorithm, the re-opening of an already discarded node is impossible; {\sc BFGP} only allows programming the maximum undefined program line that is reached after the execution of that program on all the instances in $\mathcal{P}$.
\end{proof}

\begin{theorem}[Completeness]
Given a GP problem $\mathcal{P}$, a maximum number of pointers $|Z|$, and maximum number of program lines $n$, if there is a planning program $\Pi$ within these bounds that solves $\mathcal{P}$, then the BFGP algorithm can compute it.
\end{theorem}

\begin{proof}
The {\sc BFGP} algorithm implements a complete enumeration of the entire space of planning programs with a maximum number of pointers $|Z|$ and maximum number of program lines $n$ except: (i), a search node was identified as a dead-end or (ii), the ancestor of a search node was identified as a dead-end. {\sc BFGP} is safe because it only discards a search node when its corresponding partially specified planning program failed to solve an input planning instance (which is actually the definition for not being a GP solution). Furthermore, if a partially specified planning program failed to solve an input planning instance, any planning program that can be built programming the remaining undefined program lines will also fail to solve that problem.
\end{proof}

\begin{theorem}[Soundness]
If the execution of the {\sc BFGP} algorithm on a GP problem $\mathcal{P}$ outputs a generalized plan $\Pi$, this means that $\Pi$ is a solution for $\mathcal{P}$.
\end{theorem}

\begin{proof}
The {\sc BFGP} algorithm runs until: (i) the open list is empty, which means that search space is exhausted without finding a solution and no generalized plan is output or (ii), {\sc BFGP} extracted from the open list a planning program whose execution, in all the classical planning instances $P_t\in\mathcal{P}$, resulted successful. This is actually the definition of a solution for a GP problem.
\end{proof}

\begin{theorem}[Time and Memory]
The time and memory consumption of the BFGP algorithm are characterized by the big-Oh expression  $O((|A_Z'| + (n-2) \times 4)^n)$.
\end{theorem}

\begin{proof}
The BFGP algorithm is an implementation of a BFS, whose memory and time complexity are characterized as $O(b^d)$, where $b$ denotes the branching factor and $d$ denotes the depth of the corresponding search tree. The branching factor of the search tree induced by the {\sc BFGP} algorithm is the number of different instructions that can be programmed at an undefined  program line, which is $b\leq|A_Z'| + (n-2) \times 4$; gotos can only be programmed after RAM operations. The depth of the search tree induced by the {\sc BFGP} is given by the maximum number of program lines $n$.
\end{proof}

With respect to solution quality BFGP does not guarantee that the computed planning programs are optimal. BFGP can however compute optimal planning programs when run in {\em anytime mode}. In this case we can use $f_6(\Pi,\mathcal{P})$ to rank GP solutions according to their execution cost in the classical planning instances that are comprised in the given GP problem (e.g. to prefer a sorting program with smaller computational complexity).

\section{Evaluation}
\label{sec:evaluation}
This section evaluates the empirical performance of our {\em GP as heuristic search} approach. All experiments are performed in an Ubuntu 20.04 LTS, with AMD® Ryzen 7 3700x 8-core processor $\times$ 16 and 32GB of RAM \footnote{The source code, evaluation scripts, and used benchmarks can be found at: \href{https://github.com/aig-upf/best-first-generalized-planning}{https://github.com/aig-upf/best-first-generalized-planning}.}.

\subsection{Benchmarks} 
We report results in nine different domains; two {\em propositional} domains and seven {\em numerical} domains. In the {\em propositional} domains the functions $\Phi$ that induce the state variables are Boolean. In the {\em numerical} domains these functions are positive numeric functions. Next we provide more details on the nine domains:

\begin{itemize}
    \item {\em Gripper}, a robot must pick all balls from room A and drop them in room B.
    \item {\em Visitall}, starting from the bottom-left corner of a squared grid, an agent must visit all cells.
    \item {\em Corridor}, an agent moves from an arbitrary initial location to a destination location in a corridor.
    \item {\em Fibonacci}, compute the $n^{th}$ term of the Fibonacci sequence.
    \item {\em Find}, counts the number of occurrences of a specific value in a list.
    \item {\em Reverse}, for reversing the content of a list.
    \item {\em Select}, find the minimum value of a list.
    \item {\em Sorting}, for sorting the values of a vector.
    \item {\em Triangular Sum}, compute the $n^{th}$ triangular number.
\end{itemize}

{\em Gripper} and {\em Visitall} are propositional, the remaining seven domains are numeric. For each domain, we build a GP problem that comprises ten randomly generated classical planning instances~\footnote{For reproducibility reasons we fix the random seed to generate the classical planning instances in the GP problems.}. In the case of the {\em gripper} domain, instances go from $2$ to $11$ balls in room A to be dropped in room B; in {\em visitall} instances are squared grids ranging from $2\times 2$ to $11\times 11$ cells; {\em corridor} have instances from length $3$ to $12$; {\em fibo} and {\em triangular sum} instances range from the $2^{nd}$ to the $11^{th}$ number in the sequence; last {\em find}, {\em reverse}, {\em select} and {\em sorting} have instances with vectors from size $2$ to $11$ that are initialized with random content. The result of arithmetical operations in these domains is bounded to $10^2$ in the synthesis of GP solutions, and to $10^9$ in the validation of GP solutions.

All domains include the base RAM instructions $\{{\tt\small inc}(z_1)$, ${\tt\small dec}(z_1)$, ${\tt\small cmp}(z_1,z_2)$, ${\tt\small set}(z_1,z_2)$ $| \; z_1,z_2 \in Z\}$, such that $z_1$ and $z_2$ are pointers of the same sort, and the RAM instructions  $\{{\tt\small test}(\phi(\overrightarrow{z_1})),$ ${\tt\small cmp}(\phi(\overrightarrow{z_1}),\phi(\overrightarrow{z_2})) \; | \; \overrightarrow{z_1}, \overrightarrow{z_2} \in Z^{ar(\phi)} \;  \}$, for each function $\phi\in\Phi$ and where function parameters and pointers must also be of the same sort. We recall that {\em compare} instructions are only defined for numeric functions. In addition, each domain contains the regular planning action schemes that do not affect the FLAGS.
\begin{itemize}
    \item Propositional domains. The {\em gripper} domain includes the following three action schemes; ${\tt\small move}(z_1,z_2)$ to denote the robot is moving from the room pointed by $z_2$ to the one pointed by $z_1$, ${\tt\small pick}(z_1,z_2,z_3)$ to pick the ball pointed by $z_1$, at room pointed by $z_2$, and with the gripper pointed by $z_3$, and ${\tt\small drop}(z_1,z_2,z_3)$, to drop ball $z_1$ at room $z_2$ with gripper $z_3$. {\em Visitall} only needs one action schema to visit a cell addressed by two pointers, of type row and column respectively, e.g. ${\tt\small visit}(z_1,z_2)$.
    \item Numerical domains. The {\em triangular sum} and {\em Fibonacci} domains include the action schemes ${\tt\small vector\text{-}inc}(z_1)$ and ${\tt\small vector\text{-}dec}(z_1)$, to increase and decrease by one the content of the vector at $z_1$,  and the action scheme ${\tt\small vector\text{-}add}(z_1,z_2)$ for adding the value at $z_2$ to the content of the vector at $z_1$. Similarly, {\em corridor} needs two action schemes, ${\tt\small vector\text{-}left}(z_1)$ and ${\tt\small vector\text{-}right}(z_1)$, to increase or decrease by one the location of the agent. {\em Select} only requires the base RAM instructions. {\em Find} includes the ${\tt\small accumulate}(z_1)$ action schema for counting the number of occurrences of the target element. {\em Reverse} and {\em Sorting} include the ${\tt\small swap}(z_1,z_2)$ action scheme to swap the values addressed by $z_1$ and $z_2$.
\end{itemize}
We model the regular planning actions as they are always executable but that their effects only update the current state iff their preconditions hold in the current state. Otherwise the execution of an action has no effect.

\subsection{Synthesis of GP Solutions}
Here we present two different experiments to evaluate  the performance of the {\sc BFGP} algorithm in the given benchmark. First, we asses every evaluation/heuristic function $f_i$ by running {\sc BFGP}($f_i$). Second we search for the best combination of two evaluation/heuristic functions, where one is {\em structured}-based and the other {\em performance}-based.

\subsubsection{Performance of {\sc BFGP}($f_i$)}
Table~\ref{tab:heuristics} details the results of the first synthesis experiment where the {\sc BFGP} algorithm uses each of our nine different evaluation/heuristic functions (the computation bounds are $3,600$ seconds of CPU-time and $32GB$ of memory and best results are shown in bold). Regarding {\em structure}-based functions $f_2$ dominates in all domains and metrics (except in the {\em find} domain where $f_3$ is faster) and it also has the highest coverage solving $8$ out of $9$ domains ($f_1$, $f_3$ and $f_7$ have lower coverage failing in the same three problems, namely {\em corridor}, {\em gripper} and {\em sorting}). Regarding {\em performance}-based functions, there is not a strictly dominant one since the best scores are shared among $f_4$, $f_5$ and $f_9$. Function $f_4$ has the lowest memory consumption but could not solve {\em gripper} and {\em sorting}; $f_9$ is the best for solving {\em corridor}, but it is unable to solve {\em gripper} within the time bound; and $f_5$ is the function with the least number of expanded nodes in more than half of domains, in addition to the best coverage solving all problems.

\begin{table}[]
	\centering
  \resizebox{\textwidth}{!}{ 
		\begin{tabular}{|l|c||c|c|c|c||c|c|c|c||c|c|c|c|} \hline
			\multirow{2}{*}{\textbf{Domain}} & \multirow{2}{*}{$n,|Z|$} & \multicolumn{4}{|c||}{$f_1$} & \multicolumn{4}{|c||}{$f_2$} & \multicolumn{4}{|c|}{$f_3$}  \\\cline{3-14}
			& & Time & Mem. & Exp. & Eval.  & Time & Mem. & Exp. & Eval.  & Time & Mem. & Exp. & Eval.  \\\hline
			Corridor & 8, 2 & TO & - & - & - &           1,367 & {\bf 4} & 4.2M & 4.2M &    TO & - & - & -   \\
			Fibonacci & 7, 2 & 779 & 715 & 2.9M & 5.1M & {\bf 115} & {\bf 4} & 0.5M & {\bf 0.5M} &      1,960 & 1,118 & 8.0M & 10.0M \\
			Find & 6, 3 & 32 & 18 & 0.2M & 0.2M &        31  & {\bf 4} & 0.2M & 0.2M &      {\bf 23} & 23 & 0.1M & {\bf 0.1M} \\
			Gripper & 8, 4 & TO & - & - & - &            2,968 & {\bf 4} & 13.1M & 13.1M &  TO & - & - & -   \\
			Reverse & 7, 2 & 317 & 192 & 1.4M & 2.1M &   {\bf 10} & {\bf 4} & {\bf 47.9K} & {\bf 48.0K} &     224 & 235 & 0.3M & 1.5M \\
			Select & 7, 2 &  192 & 96 & 0.8M & 1.1M &    {\bf 15} & {\bf 4} & 82.8K & 82.8K &     98 & 97 & 0.2M & 0.6M   \\
			Sorting & 11, 2 & TO & - & - & - &           TO & - & - & - &             TO & - & - & -    \\
			T. Sum & 6, 2 & 40 & 58 & 0.2M & 0.4M &      {\bf 2} & {\bf 4} & {\bf 14.7K} & {\bf 14.7K} &      38 & 75 & 61.2K & 0.4M  \\
			Visitall & 8, 2 & 1,631 & 219 & 2.0M & 2.8M &38 & {\bf 4} & 66.1K & 66.2K &     408 & 112 & 0.2M & 0.7M\\\hline
		    \multicolumn{2}{|c||}{Average} & 498.5 & 216.3 & 1.2M & 1.9M & 568.3 & {\bf 4.0} & 2.3M & 2.3M & 458.5 & 276.7 & 1.5M & 2.2M \\\hline\hline
			 &  & \multicolumn{4}{|c||}{$f_4$} & \multicolumn{4}{|c||}{$f_5$} & \multicolumn{4}{|c|}{$f_6$}  \\\cline{3-14}
			Corridor & 8, 2 &   1,521 & 5 & 4.2M & 4.2M &   970 & 80 & 1.9M & 2.1M &    TO & - & - & -  \\    
			Fibonacci & 7, 2 &  147 & {\bf 4} & 0.5M & {\bf 0.5M} &     206 & 203 & 0.2M & 1.2M &   2,798 & 1,779 & 11.0M & 11.0M\\
			Find & 6, 3 &       39 & {\bf 4} & 0.2M & 0.2M &      34 & 23 & {\bf 82.9K} & 0.2M &    41 & 31 & 0.2M & 0.2M \\
			Gripper & 8, 4 &    TO & - & - & - &            {\bf 10} & 18 & {\bf 9.9K} & {\bf 72.9K} &    TO & - & - & - \\
			Reverse & 7, 2 &    13 & {\bf 4} & 48.3K & 48.4K &    690 & 356 & 2.5M & 2.5M &   651 & 380 & 2.5M & 2.5M \\
			Select & 7, 2 &     21 & {\bf 4} & 85.5K & 86.5K &    17 & 12 & {\bf 43.1K} & {\bf 68.3K} &   228 & 171 & 1.0M & 1.1M \\
			Sorting & 11, 2 &   TO & - & - & - &            {\bf 2,693} & 110 & {\bf 1.5M} & {\bf 1.5M} & TO & - & - & - \\
			T. Sum & 6, 2 &     3 & {\bf 4} & {\bf 14.7K} & {\bf 14.7K} &     15 & 7 & 72.6K & 78.1K &    84 & 99 & 0.6M & 0.6M \\
			Visitall & 8, 2 &   53 & 5 & 67.5K & 68.1K &    {\bf 3} & 5 & {\bf 683} & {\bf 2.0K}  &       2,474 & 365 & 2.8M & 2.8M \\\hline
		    \multicolumn{2}{|c||}{Average} & {\bf 61.1} & 4.3 & {\bf 0.7M} & {\bf 0.7M} & 515.3 & 90.4 & {\bf 0.7M} & 0.9M & 1,046.0 & 470.8 & 3.0M & 3.0M \\\hline\hline
			 &  & \multicolumn{4}{|c||}{$f_7$} & \multicolumn{4}{|c||}{$f_8$} & \multicolumn{4}{|c|}{$f_9$}  \\\cline{3-14}
			Corridor & 8, 2 &  TO & - & - & - &             1,317 & 368 & 2.1M & 2.8M &     {\bf 857} & 94 & {\bf 1.8M} & {\bf 2.0M}  \\
			Fibonacci & 7, 2 & 789 & 716 & 2.9M & 5.1M &    251 & 264 & 0.2M & 1.5M &       157 & 195 & {\bf 0.1M} & 1.1M \\
			Find & 6, 3 &      32 & 18 & 0.2M & 0.2M &      39 & 22 & 0.2M & 0.2M &         34 & 21 & 0.2M & 0.2M \\
			Gripper & 8, 4 &   TO & - & - & - &             TO & - & - & - &                TO & - & - & - \\
			Reverse & 7, 2 &   336 & 197 & 1.5M & 2.1M &    662 & 339 & 2.5M & 2.5M &       677 & 339 & 2.5M & 2.5M \\
			Select & 7, 2 &    200 & 99 & 0.8M & 1.1M &     225 & 123 & 1.0M & 1.1M &       17 & 12 & 57.2K & 68.4K \\
			Sorting & 11, 2 &  TO & - & - & - &             2,711 & {\bf 95} & {\bf 1.5M} & {\bf 1.5M} &      2,820 & {\bf 95} & {\bf 1.5M} & {\bf 1.5M} \\
			T. Sum & 6, 2 &    39 & 58 & 0.2M & 0.4M &      15 & 8 & 72.6K & 78.1K &        15 & 8 & 72.6K & 78.1K \\
			Visitall & 8, 2 & 846 & 196 & 1.2M & 1.7M &     2,714 & 344 & 2.8M & 2.8M &     60 & 10 & 41.2K & 67.3K \\\hline
		    \multicolumn{2}{|c||}{Average} & 373.7 & 214.0 & 1.1M & 1.8M & 991.8 & 195.4 & 1.3M & 1.6M & 579.6 & 96.8 & 0.8M & 0.9M \\\hline
		\end{tabular}
		}
	\caption{ We report the number of program lines $n$, and pointers $|Z|$ per domain, and for each evaluation/heuristic function, CPU (secs), memory peak (MBs), and the numbers of expanded and evaluated nodes. TO stands for Time-Out ($>$1h of CPU). Best results in bold.}
	\label{tab:heuristics}
\end{table}

Table~\ref{tab:heuristics_summary} summarizes the results from Table~\ref{tab:heuristics}, grouping results by domains and averaging the metrics by the total number of functions that solved each domain. There are 6 domains which are solved by all the nine evaluation/heuristic functions. In the rest of domains, there are at least 4 or more functions that do not solve them, such as {\em gripper} which is the least solved domain (only $f_2$ and $f_5$ solve it), and {\em sorting} which is solved by $f_5$, $f_8$ and $f_9$ but it is the hardest in terms of time average.

\begin{table}[]
    \centering
    \begin{tabular}{|l||c|c|c|c|c|} \hline
         {\bf Domain} & Time & Mem. & Exp. & Eval. & \#$f_i$ Solved \\\hline
         Corridor &  1,206.4 & 110.2 & 2.8M & 3.1M & 5\\
         Fibonacci & 800.2 & 555.3 & 2.9M & 4.0M & 9\\
         Find & 33.9 & 18.2	& 0.2M	 & 0.2M & 9 \\
         Gripper & 1,489.0 &	11.0	& 6.6M &	6.6M &	2\\
         Reverse & 397.8 &	227.3 &	1.5M &	1.8M & 	9\\
         Select & 112.6 &	68.7 &	0.5M &	0.6M &	9\\
         Sorting & 2,741.3	& 100.0 &	1.5M &	1.5M &	3\\
         T. Sum & 27.9 &	35.7	& 0.1M	& 0.2M &	9\\
         Visitall & 914.1 &	140.0 &	1.1M &	1.2M &	9\\\hline
    \end{tabular}
    \caption{We report for each domain, the time (secs), memory peak (MBs), and expanded and evaluated nodes averaged by the number of functions that solved the domain in Table~\ref{tab:heuristics}.}
    \label{tab:heuristics_summary}
\end{table}

\subsubsection{The synthesized solutions}
Figure~\ref{fig:synthesis} shows the programs computed by $BFGP(f_5)$. In {\em Corridor} there are two pointers, $i$ pointing to the agent location and $j$ pointing to the target location; the solution moves the agent right until it surpasses the target location, then it moves the agent left until it reaches the target location. In {\em Fibonacci}, pointers $a$ and $b$ are used to compute the $n$-th Fibonacci number, where $a$ addresses the $F_a$ number to which $F_{a-1}$ and $F_{a-2}$ are added using $b$ as an auxiliary pointer; and finishes when $a$ reaches the $n$-th element (the last one). In {\em Find}, there is a pointer $i$ to iterate over a vector, a pointer $t$ which targets a value in the vector, and a counter pointer $a$ whose address content is {\em only} increased by one every time an occurrence of the target $t$ is found in the vector (Lines 1-2). The process repeats until $i$ reaches the end of the vector. 

The {\em Gripper} solution uses one pointer for balls ($b_1$), two for rooms ($r_1$ and $r_2$) and one for grippers $g_1$; for each ball $b_1$, the agent will pick it up from room $r_1$ (always room A) using gripper $g_1$ (always left gripper), sets $r_2$ to room B, moves from A to B, drop ball $b_1$ at room B, goes back to room A, and continues with the next ball. The {\em Reverse} domain uses two pointers $i$ and $j$ and finds a solution with $O(n^2)$ complexity of a vector of size $n$; it moves all values from $j$ to $n-1$ indexes one location to right and places the last element in the $j$-th location, using $i$ as an auxiliary pointer; then increases $j$ by one until it reaches the end of the vector. The {\em Select} domain has two pointers $a$ and $b$; it iterates over the vector with pointer $a$, and assigns $a$ to $b$ every time the value pointed by $a$ is smaller than the one pointed by $b$. 

The {\em Sorting} solution is  succinct but complex to interpret; $j$ always points to the first location, so all swaps are done with this location; then, two cases may happen when reaching Line 3: either the first element was wrongly sorted in the previous swap and detected because the $i$-th value is larger than the first, so all values from the $i$-th to the $n$-th location are shifted one place to their right, the first element is placed in the $i$-th location (now is correctly order with respect to $i+1$ value) and the last is placed first (defined in Lines 0-6), and continues again in Line 3; or the largest value of the vector is in the first location and the rest are sorted in increasing order, so the problem can be solved by shifting all values once to their left applying instructions at Lines 3, 4, 7, 8, and 9 in sequence. In {\em Triangular Sum}, $b$ points to the $i^{th}$ number in the sequence and $a$ to value $n$, then $a$ is added to $b$, $a$ is decreased by one, and the whole process repeats until the value $a$ is pointing is 0. The last domain, {\em Visitall}, has two pointers $i$ and $j$ that are used for rows and columns, respectively; since the agent always starts in the bottom-left corner, it visits all $i$-th cells for a given $j$; then it moves back to the first row, increases the columns ($j$) by one, and repeats the procedure until all columns have been explored.

\begin{figure*}
	\begin{scriptsize}
		\begin{subfigure}[t]{0.30\textwidth}
			\begin{lstlisting}[mathescape]
    CORRIDOR
    0. vector-right(i)
    1. inc(j)
    2. cmp(vector(i),vector(j))
    3. goto(0,$\neg$($\neg y_z\wedge y_c$))
    4. vector-left(i)
    5. cmp(vector(i),vector(j))
    6. goto(1,$\neg$($y_z\wedge\neg y_c$))
    7. end
			
			
			
			\end{lstlisting}
		\end{subfigure}	
		\hspace{.5cm}
		\begin{subfigure}[t]{0.30\textwidth}
			\begin{lstlisting}[mathescape]
    FIBONACCI		
    0. vector-add(a,b)
    1. dec(b)
    2. vector-add(a,b)
    3. set(b,a)
    4. inc(a)
    5. goto(0,$\neg$($ y_z\wedge \neg y_c$))
    6. end
			
			\end{lstlisting}
		\end{subfigure}
		\hspace{.5cm}		
		\begin{subfigure}[t]{0.30\textwidth}
			\begin{lstlisting}[mathescape]
    FIND
    0. cmp(vector(i),vector(t))
    1. goto(3,$\neg ( y_z \wedge \neg y_c)$)
    2. accumulate(a)
    3. inc(i)
    4. goto(0,$\neg ( y_z \wedge \neg y_c)$)
    5. end
			
			\end{lstlisting}
		\end{subfigure}	
		\newline
		\begin{subfigure}[t]{0.30\textwidth}
			\begin{lstlisting}[mathescape]
    GRIPPER
    0. pick(b1,r1,g1)
    1. inc(r2)
    2. move(r1,r2)
    3. drop(b1,r2,g1)
    4. move(r2,r1)
    5. inc(b1)
    6. goto(0,$\neg (y_z \wedge \neg y_c)$)
    7. end
			
			\end{lstlisting}
		\end{subfigure}
		\hspace{.5cm}
		\begin{subfigure}[t]{0.30\textwidth}
			\begin{lstlisting}[mathescape]
    REVERSE
    0. set(i,j)
    1. swap(i,j)
    2. inc(i)
    3. goto(1,$\neg$($y_z\wedge\neg y_c$))
    4. inc(j)
    5. goto(0,$\neg$($y_z\wedge\neg y_c$))
    6. end

			\end{lstlisting}
		\end{subfigure}	
		\hspace{.5cm}
		\begin{subfigure}[t]{0.30\textwidth}
			\begin{lstlisting}[mathescape]
    SELECT
    0. inc(b)
    1. cmp(vector(a),vector(b))
    2. goto(4,$\neg$($\neg y_z\wedge\neg y_c$))
    3. set(b,a)
    4. inc(a)
    5. goto(1,$\neg$($y_z\wedge\neg y_c$))
    6. end
    
			\end{lstlisting}
		\end{subfigure}	
		\newline 
		\begin{subfigure}[t]{0.30\textwidth}
			\begin{lstlisting}[mathescape]
    SORTING
    0. swap(i,j)
    1. inc(i)
    2. goto(0,$\neg$($y_z\wedge\neg y_c$))
    3. cmp(vector(i),vector(j))
    4. goto(7,$\neg$($\neg y_z\wedge y_c$))
    5. dec(i)
    6. goto(0,$\neg$($y_z\wedge y_c$))
    7. swap(i,j)
    8. dec(i)
    9. goto(3,$\neg$($y_z\wedge\neg y_c$))
    10. end
			\end{lstlisting}
		\end{subfigure}
		\hspace{.5cm}
		\begin{subfigure}[t]{0.30\textwidth}
			\begin{lstlisting}[mathescape]
    TRIANGULAR SUM
    0. inc(a)
    1. vector-add(b,a)
    2. vector-dec(a)
    3. test(vector(a))
    4. goto(0,$\neg$($y_z\wedge\neg y_c$))
    5. end

			\end{lstlisting}
		\end{subfigure}	
		\hspace{.5cm}
		\begin{subfigure}[t]{0.30\textwidth}
			\begin{lstlisting}[mathescape]
    VISITALL
    0. visit(i,j)
    1. inc(i)
    2. goto(0,$\neg$($y_z\wedge\neg y_c$))
    3. dec(i)
    4. goto(3,$\neg$($y_z\wedge\neg y_c$))
    5. inc(j)
    6. goto(0,$\neg$($y_z\wedge\neg y_c$))
    7. end
    
			\end{lstlisting}
		\end{subfigure}	
	\end{scriptsize}
	\caption{Solutions computed by {\sc BFGP}$(f_5)$.}
	\label{fig:synthesis}
\end{figure*}

\subsubsection{Validation of the synthesized solutions}
Here we validate the {\sc BFGP}($f_5$) solutions of Figure~\ref{fig:synthesis} with a larger and harder set of  instances. Table~\ref{tab:validation} reports the CPU time, and peak memory, yield when running the solutions synthesized by $BFGP(f_5)$ on a validation set. All the solutions synthesized by $BFGP(f_5)$ were successfully validated, besides {\em Reverse} with infinite detection mode that ran out of memory. The largest CPU-times and memory peaks correspond to the configuration that implements the detection of {\em infinite programs}, which requires saving states to detect whether they are revisited during execution. Skipping this mechanism allows to validate terminating programs much faster~\cite{aguas2020generalized}.

In the validation set, each state variable from the planning domain is bounded by $10^9$, instead of $10^2$ which was the synthesis bound. {\em Corridor} and {\em Gripper} are validated over $1{,}000$ instances, where for each $n \in [12,1{,}011]$, the first has random initial and goal locations below $n$, and the second $n$ balls initially in room A. {\em Fibonacci} has a validation set of $33$ instances, ranging from the $11^{th}$ Fibonacci term to the $43^{rd}$, i.e.~the integer $701{,}408{,}733$ (the $44^{th}$ number would overflow the validation bound). The solutions for {\em Reverse}, {\em Select}, and {\em Find} domains, are validated on $102$ instances each, with vector sizes ranging from $1{,}000$ to $11{,}100$, and random integer elements bounded by $10^9$. Similarly, {\em Sorting} has $100$ validation instances with vectors of random integers, but their sizes range from $12$ to $111$. The solution for {\em Triangular Sum} is validated over $44{,}709$ instances, the last one corresponding to the $44{,}720^{th}$ term in the sequence, i.e.~the integer $999{,}961{,}560$ (as in {\em Fibonacci}, the next number would overflow). In {\em Visitall}, there are $50$ validation instances with squared grids range from $12\times 12$ to $61\times 61$. 

\begin{table}
	\centering
	\begin{small}
		\begin{tabular}{|l|r||r|r||r|r|}\hline
			\textbf{Domain}  & \textbf{Instances} & Time$_\infty$ & Mem$_\infty$ & Time & Mem \\\hline
			Corridor & 1,000 & 0.54 & 6.5MB & \textbf{0.43} & {\bf 6.3MB} \\
			Fibonacci & 33 & \textbf{0.01} & 4.8MB & \textbf{0.01} & \textbf{4.7MB} \\
			Find & 102 & 1,542.85 & 2.2GB & \textbf{1.77} & \textbf{0.3GB}\\
			Gripper & 1,000 & 93.5 & \textbf{0.5GB} & \textbf{4.77} & \textbf{0.5GB} \\
			Reverse & 102 & ME & ME & \textbf{3,553.16} & \textbf{0.5GB} \\
			Select & 102 & 1,407.54 & 2.5GB & \textbf{1.87} & \textbf{0.3GB} \\
			Sorting & 100 & 230.98 & 0.5GB & \textbf{17.19} & \textbf{8.8MB} \\
			Triangular Sum & 44,709 & 3,244.84 & \textbf{0.1GB} & \textbf{2,357.56} & \textbf{0.1GB} \\
			Visitall & 50 & 42.21 & 0.2GB & \textbf{0.33} & \textbf{48MB} \\  \hline
		\end{tabular}
	\end{small}
	\caption{Validation set, CPU-time (secs) and memory peak for program validation, with/out {\em infinite program} detection. ME stands for memory exceeded. Best results in bold.}
	\label{tab:validation}
\end{table}

\subsubsection{Performance of {\sc BFGP} with function combinations}
Interestingly, the base performance of {\sc BFGP} with a single evaluation/heuristic function is improved combining both structural and cost-to-go information; we can guide the search of {\sc BFGP} with a cost-to-go heuristic function and break ties with a structural evaluation function, and vice versa. Thus, we run all configurations of {\sc BFGP}($f_i,f_j$) and {\sc BFGP}($f_j,f_i$) such that $f_i\in \{f_1,f_2,f_3,f_7\}$ and $f_j\in\{f_4,f_5,f_6,f_8,f_9\}$, and select the configuration that solves all domains and with the best average time. There are $40$ {\sc BFGP}($f_i,f_j$)/{\sc BFGP}($f_j,f_i$) configurations, but only {\sc BFGP}($f_5,f_3$) and {\sc BFGP}($f_5,f_7$) are able to solve all domains. The performance of these two configurations is then compared against {\sc BFGP}($f_5$), since it is the only single evaluation/heuristic function that solve all domains in the previous experiment. Table~\ref{tab:h-combined} summarizes that comparison, showing that {\sc BFGP}($f_5$) is improved in every domain either by {\sc BFGP}($f_5,f_3$) or {\sc BFGP}($f_5,f_7$). Furthermore, {\sc BFGP}($f_5,f_7$) has the best average performance in all four metrics, empirically proving that combining goal-oriented functions with {\em structural}-based functions that measure the syntactic complexity of a program, in that specific order, is the best configuration. 

We also compared the performance of $BFGP (f_5,f_7)$, in terms of CPU-time, with the compilation-based approach for GP~\cite{segovia2016generalized,segovia2019computing}. The compilation-based approach, that we named {\sc PP}, computes planning programs, following a top-down strategy, with the planner {\sc LAMA-2011} (first solution setting) to solve the classical planning problems that result from the compilation. Table~\ref{tab:comparisson} summarizes the results of this comparison. There are 3 domains where PP is faster than {\sc BFGP}($f_5,f_7$), but in these domains the GP problems addressed by {\sc PP} are easier: i) {\em Gripper} in PP has the same action move for both directions and picks are only available for the next ball, while in {\sc BFGP} actions are parameterized with pointers, so it first needs to find that pointers $r_1$ and $r_2$ point to rooms A and B respectively, pick balls only from room $r_1$, move from $r_1$ to $r_2$ to drop the ball, and move backwards from $r_2$ to $r_1$; ii) {\em Reverse} in PP has one of the pointers in the last position of vector from the initial state, reproducing this setting in {\sc BFGP} a program of $6$ lines is found in less than 1 second after expanding $260$ nodes and evaluating $2\text{,}560$ nodes, however, it is more interesting to us a solution that synthesizes where to place and how to use each pointer, even though it is a harder problem; and iii) {\em Triangular Sum} in PP just accumulates one variable to another one, while in {\sc BFGP} the pointers should point to the right variables, then use them. In the rest of domains, {\sc BFGP} dominates PP, even though programs are larger, {\sc BFGP} must reason on how to use the pointers, and {\sc BFGP} uses more instances with larger values, being able to solve domains where PP dies because of the grounding among other reasons.

\begin{table}[]
    \centering
  \resizebox{\textwidth}{!}{ 
    \begin{tabular}{|l||c|c|c|c||c|c|c|c||c|c|c|c|} \hline
        \multirow{2}{*}{\bf Domain} & \multicolumn{4}{|c||}{BFGP$(f_5,f_3)$} & \multicolumn{4}{|c||}{BFGP$(f_5,f_7)$} & \multicolumn{4}{|c|}{BFGP$(f_5)$}\\\cline{2-13} 
          & Time  & Mem. & Exp. & Eval. &  Time  & Mem. & Exp. & Eval. & Time  & Mem. & Exp. & Eval.  \\\hline
            Corridor &  675 & {\bf 55} & 1.8M & 2.0M & {\bf 658}  & 64 & {\bf 1.7M} & {\bf 1.9M} & 970 & 80 & 1.9M & 2.1M  \\
			Fibonacci & 990 & 553 & 2.7M & 4.3M & {\bf 55} & {\bf 68} & {\bf 43.4K} & {\bf 0.4M} & 206 & 203 & 0.2M & 1.2M \\
			Find  & {\bf 29} & 16 & {\bf 68.5K} & {\bf 0.1M} & 38 & {\bf 14} & 0.2M & 0.2M & 34 & 23 & 82.9K & 0.2M      \\
			Gripper  & 9 & 17 & 9.7K & 67.4K & {\bf 7} & {\bf 13} & {\bf 8.7K} & {\bf 50.1K} & 10 & 18 & 9.9K & 72.9K  \\
			Reverse  & 702 & 217 & {\bf 2.5M} & {\bf 2.5M} & {\bf 676} & {\bf 184} & {\bf 2.5M} & {\bf 2.5M} & 690 & 356 & {\bf 2.5M} & {\bf 2.5M}   \\
			Select  & {\bf 14} & 10 & {\bf 32.0K} & {\bf 54.0K} & 17 & {\bf 9} & 47.6K & 68.3K & 17 & 12 & 43.1K & 68.3K     \\
			Sorting  & {\bf 2,484} & {\bf 39} & {\bf 1.3M} & {\bf 1.4M} & 2,710  & 82 & 1.5M & 1.5M & 2,693 & 110 & 1.5M & 1.5M  \\
			T. Sum  & {\bf 15} & 7 & {\bf 72.6K} & {\bf 78.0K} & {\bf 15} & {\bf 6} & {\bf 72.6K} & 78.1K & {\bf 15} & 7 & {\bf 72.6K} & 78.1K     \\
			Visitall  & {\bf 2} & {\bf 5} & {\bf 275} & {\bf 605} & 3 & {\bf 5} & 582 & 1.7K & 3 & {\bf 5} & 683 & 2.0K \\\hline
		    Average & 546.7 & 101.9 & 0.9M & 1.2M & {\bf 464.3} & {\bf 49.4} & {\bf 0.7M} & {\bf 0.8M} & 579.6 & 96.8 & 0.8M & 0.9M \\\hline
    \end{tabular} 
		}
    \caption{For each domain we report, CPU time (secs), memory peak (MBs), num. of expanded and evaluated nodes of {\sc BFGP}($f_5,f_3$), {\sc BFGP}($f_5,f_7$) and {\sc BFGP}($f_5$). TO means time-out ($>$ 1h of CPU). Best results in bold.}
    \label{tab:h-combined}
\end{table}

\begin{table}[]
    \centering
    \small
    \begin{tabular}{|l|c|c|} \hline
        {\bf Domain}& PP in sec. & BFGP$(f_5,f_7)$ in sec. \\ \hline 
        Corridor  & - & {\bf 658} \\
        Fibonacci  & 3,570 & {\bf 55} \\
        Find  & 274.86 & {\bf 38} \\
        Gripper  & {\bf 1} & 7 \\ 
        Reverse  & {\bf 87.86} & 676 \\
        Select  & 204.20 & {\bf 17} \\
        Sorting & - & {\bf 2,710} \\
        Triangular Sum & {\bf 0.85} & 15 \\
        Visitall & - & {\bf 3} \\\hline
    \end{tabular}
    \caption{Computing CPU-time (secs) for solving domains in the GP compilation approach (PP) and $BFGP (f_5,f_7)$. }
    \label{tab:comparisson}
\end{table}

\subsection{Validation of GP solutions for more complex domains}
Here we present several GP benchmarks, with known polynomial time solutions, but that result too complex for our current {\sc BFGP} algorithm (within the given time and memory bounds). Our aim is showing that our approach is expressive enough to represent solutions to GP problems coming from IPC planning domains, noise-free supervised classification tasks, and numerical domains. These solutions are succinctly represented as planning programs, instead of long sequences of grounded actions for large problems, and validated efficiently without being affected by the grounding methods of planners. 

\begin{itemize}
    \item {\em Blocks Ontable},  towers of blocks where all blocks must be placed on the table. 
    \item {\em Grid}, an agent has to move from an arbitrary location to a destination one in a 2D grid.
    \item {\em Miconic}, is an elevator problem where passengers at origin, wait for the elevator to enter, and then served at their destination floor.
    \item {\em Michalski Trains}, is a classic of relational supervised machine learning. A binary noise-free classification task with $10$ trains that either go east or west, and multiple features such as the number of wheels, wagons, or their shape for each train among others. The goal is to learn the features that classify all trains in the right direction.
    \item {\em Satellite}, consists of taking images of different targets with instruments that are boarded in satellites. In addition, instruments need to be calibrated and in specific modes for taking each image; and each satellite has only power for one instrument at a time, so it needs to switch the current instrument off, switch on the next and calibrate it, before using a new instrument for taking images.
    \item {\em Sieve of Erathostenes}, is a method to find prime numbers up to a certain bound using only additive and iterative mechanisms.
    \item {\em Spanner}, consists of tighten all loose nuts at the end of a corridor, with the picked spanners along the corridor. Spanners can only be used once, and when the agent moves to the next room it can not go back, so if there are unpicked spanners in visited rooms the task could become unsolvable.
\end{itemize}

Figure~\ref{fig:complex_solutions} shows the hand-coded solutions for these benchmarks. In {\em Blocks Ontable}, given $n$ blocks the complexity of the solution is cubic, i.e. $O(n^3)$, where it searches $n$ times, every $o1$ block that is on top of an $o2$ block, then unstack and put $o1$ down on the table. In {\em Grid}, the agent moves to the bottom left corner, then each coordinate is increased by one while they are smaller than their goal, visiting the resulting coordinate. In {\em Spanner}, an agent picks up all available spanners in location $l1$, walks to the next $l2$ location and repeats the process until it reaches the last location (the gate), collecting all spanners on its path; once in the gate, it tightens each loose nut with a spanner. The solution to {\em Michalski Trains} is summarized as, each train $t1$ will go east if it has a car which is closed and short, otherwise it will go west. In {\em Sieve of Eratosthenes} all numbers are initially classified as primes, and it should decide whether they are not; so it iterates over $i$ and uses $j$ and $k$ as auxiliary pointers, where the first acts as a counter that ranges from $0$ to $i$, and second adds up to the next multiple of $i$, i.e. $k\;\%\;i = 0$; then every $k$-th number will be set to no prime, $i$ is increased by one and the process repeats until $i$ reaches the last element. In {\em Miconic}, the elevator always starts in the first floor $f1$, so for every floor it boards and departs passenger $p1$ whether possible; once it reaches the last floor, all passengers are either served or in the elevator, so it will serve all possible passenger in each floor while it goes down, until the first floor is reached again. The last domain, {\em Satellite}, is the most complex because it requires iteration over multiple variable types, i.e. satellites, instruments, modes and directions. The solution to this domain consists of switching off all instruments and turning all satellites to the first direction; then for each satellite, the $i1$ instrument is switched on, calibrated with its calibration target direction $d2$, and used to take images of every direction $d2$ in every mode $m1$; once it finishes, the satellite turns to the first direction $d1$ again, switches off the current instrument, and continues with the next one, until all satellites have used all their instruments.

We get one main take away lesson from the analysis of Figure~\ref{fig:complex_solutions} solutions; solutions have common high-level structures, that either iterate over all combinations of variable types (i.e. {\em Blocks}, {\em Miconic}, {\em Satellite}, \ldots) or build a complex logic query (i.e. {\em Michalski Trains}). This suggests that planning programs may be synthesized more efficiently using predefined structures (such as {\tt FOR} or {\tt IF-THEN-ELSE} constructs) although this is out of the scope of this paper.

\begin{figure}[hbt!]
    \centering
    \begin{scriptsize}
	\begin{subfigure}[t]{0.33\textwidth}
	\begin{lstlisting}[mathescape]    
    BLOCKS ONTABLE
    0. dec(o2)
    1. goto(0,$\neg$($y_z\wedge \neg y_c$))
    2. dec(o1)
    3. goto(2,$\neg$($y_z\wedge \neg y_c$))
    4. unstack(o1,o2)
    5. put-down(o1)
    6. inc(o1)
    7. goto(4,$\neg$($y_z\wedge \neg y_c$))
    8. inc(o2)
    9. goto(2,$\neg$($y_z\wedge \neg y_c$))
    10. inc(o3)
    11. goto(0,$\neg$($y_z\wedge \neg y_c$))
    12. end
    \end{lstlisting}
    \end{subfigure}
	\begin{subfigure}[t]{0.33\textwidth}
	\begin{lstlisting}[mathescape]    
    GRID
    0. dec(i)
    1. goto(0,$\neg$($y_z\wedge \neg y_c$))
    2. dec(j)
    3. goto(0,$\neg$($y_z\wedge \neg y_c$))
    4. test(goal-xpos(i))
    5. goto(8,$\neg$($y_z\wedge \neg y_c$))
    6. inc(i)
    7. goto(4,$\neg$($y_z\wedge \neg y_c$))
    8. test(goal-ypos(j))
    9. goto(12,$\neg$($y_z\wedge \neg y_c$))
    10. inc(j)
    11. goto(8,$\neg$($y_z\wedge \neg y_c$))
    12. visit(i,j)
    13. end
    \end{lstlisting}
    \end{subfigure}
	\begin{subfigure}[t]{0.33\textwidth}
	\begin{lstlisting}[mathescape]    
    SPANNER
    0. walk(l1,l2,m1)
    1. set(l1,l2)
    2. pickup_spanner(l1,s1,m1)
    3. inc(s1)
    4. goto(2,$\neg$($y_z\wedge \neg y_c$))
    5. dec(s1)
    6. goto(5,$\neg$($y_z\wedge \neg y_c$))
    7. inc(l2)
    8. goto(0,$\neg$($y_z\wedge \neg y_c$))
    9. tighten_nut(l1,s1,m1,n1)
    10. inc(s1)
    11. inc(n1)
    12. goto(9,$\neg$($y_z\wedge \neg y_c$))
    13. end
    \end{lstlisting}
    \end{subfigure}
    \newline
	\begin{subfigure}[t]{0.33\textwidth}
	\begin{lstlisting}[mathescape]    
    MICHALSKI TRAINS
    0. test(hascar(t1,c1))
    1. goto(7,$\neg$($\neg y_z\wedge y_c$))
    2. test(closed(c1))
    3. goto(7,$\neg$($\neg y_z\wedge y_c$))
    4. test(short(c1))
    5. goto(7,$\neg$($\neg y_z\wedge y_c$))
    6. set-eastbound(t1)
    7. inc(c1)
    8. goto(0,$\neg$($y_z\wedge \neg y_c$))
    9. set-westbound(t1)
    10. dec(c1)
    11. goto(10,$\neg$($y_z\wedge \neg y_c$))
    12. inc(t1)
    13. goto(0,$\neg$($y_z\wedge \neg y_c$))
    14. end

    \end{lstlisting}
    \end{subfigure}
	\begin{subfigure}[t]{0.33\textwidth}
	\begin{lstlisting}[mathescape]    
    SIEVE OF ERATHOSTENES
    0. inc(i)
    1. inc(i)
    2. set(k,i)
    3. dec(j)
    4. goto(3,$\neg$($y_z\wedge \neg y_c$))
    5. inc(k)
    6. goto(13,$\neg$($\neg y_z\wedge y_c$))
    7. inc(j)
    8. cmp(i,j)
    9. goto(5,$\neg$($y_z\wedge \neg y_c$))
    10. set-no-prime(k)
    11. cmp(i,j)
    12. goto(3,$\neg$($\neg y_z\wedge y_c$))
    13. inc(i)
    14. goto(2,$\neg$($y_z\wedge \neg y_c$))
    15. end

    \end{lstlisting}
    \end{subfigure}
	\begin{subfigure}[t]{0.33\textwidth}
	\begin{lstlisting}[mathescape]    
    MICONIC
    0. board(p1,f1)
    1. depart(p1,f1)
    2. inc(p1)
    3. goto(0,$\neg$($y_z\wedge \neg y_c$))
    4. dec(p1)
    5. goto(4,$\neg$($y_z\wedge \neg y_c$))
    6. inc(f2)
    7. up(f1,f2)
    8. inc(f1)
    9. goto(0,$\neg$($y_z\wedge \neg y_c$))
    10. dec(f2)
    11. down(f1,f2)
    12. depart(p1,f2)
    13. inc(p1)
    14. goto(12,$\neg$($y_z\wedge \neg y_c$))
    15. dec(p1)
    16. goto(15,$\neg$($y_z\wedge \neg y_c$))
    17. dec(f1)
    18. goto(10,$\neg$($y_z\wedge \neg y_c$))
    19. end

    \end{lstlisting}
    \end{subfigure}
    \newline 
	\begin{subfigure}[t]{0.33\textwidth}
	\begin{lstlisting}[mathescape]    
    SATELLITE
    0. switch_off(i1,s1)
    1. inc(i1)
    2. goto(0,$\neg$($y_z\wedge \neg y_c$))
    3. dec(i1)
    4. goto(3,$\neg$($y_z\wedge \neg y_c$))
    5. turn_to(s1,d1,d2)
    6. inc(d2)
    7. goto(5,$\neg$($y_z\wedge \neg y_c$))
    8. set(d2,d1)
    9. inc(s1)
    10. goto(0,$\neg$($y_z\wedge \neg y_c$))
    11. dec(s1)
    12. goto(11,$\neg$($y_z\wedge \neg y_c$))
    \end{lstlisting} 
    \end{subfigure} 
    \begin{subfigure}[t]{0.33\textwidth}
    \begin{lstlisting}[mathescape]
    
    13. switch_on(i1,s1)
    14. test(cal_target(i1,d2))
    15. goto(19,$\neg$($\neg y_z\wedge y_c$))
    16. turn_to(s1,d2,d1)
    17. calibrate(s1,i1,d2)
    18. turn_to(s1,d1,d2)
    19. inc(d2)
    20. goto(14,$\neg$($y_z\wedge \neg y_c$))
    21. set(d2,d1)
    22. take_image(s1,d2,i1,m1)
    23. inc(m1)
    24. goto(22,$\neg$($y_z\wedge \neg y_c$))
    25. dec(m1)
    26. goto(25,$\neg$($y_z\wedge \neg y_c$))        
    \end{lstlisting}
    \end{subfigure}
    \begin{subfigure}[t]{0.33\textwidth}
    \begin{lstlisting}[mathescape]
    
    27. inc(d2)
    28. turn_to(s1,d2,d1)
    29. inc(d1)
    30. goto(22,$\neg$($y_z\wedge \neg y_c$))
    31. dec(d1)
    32. goto(31,$\neg$($y_z\wedge \neg y_c$))
    33. turn_to(s1,d1,d2)
    34. set(d2,d1)
    35. switch_off(i1,s1)
    36. inc(i1)
    37. goto(13,$\neg$($y_z\wedge \neg y_c$))
    38. dec(i1)
    39. goto(38,$\neg$($y_z\wedge \neg y_c$))
    40. inc(s1)
    41. goto(13,$\neg$($y_z\wedge \neg y_c$))
    42. end
    \end{lstlisting}
    \end{subfigure}
    \end{scriptsize}
    \caption{Solutions to complex domains.}
    \label{fig:complex_solutions}
\end{figure}

\newpage
Table~\ref{tab:complex-validation} shows the validation results in complex domains, where validation without infinite detection scales much better again, and all domains are successfully validated (besides {\em Satellite} with infinite detection mode that runs out of memory). {\em Blocks Ontable} can be solved with $13$ lines and $3$ pointers, and the validation set consists of $20$ instances that range from $12$ to $31$ blocks. {\em Grid} requires $14$ lines of code and $2$ pointers, and it is validated with $248$ instances with grids between $5\times 5$ and $66\times 66$ size. {\em Miconic} needs $20$ lines and $3$ pointers, and $20$ instances that validates from $12$ floors and $18$ passengers to $31$ floors and $46$ passengers. {\em Michalski Trains} uses $15$ lines and $6$ pointers to classify all the trains in the unique classical task with $10$ trains and their features. {\em Satellite} is by difference the most complex in terms of required lines and pointers, which are $43$ and $5$, respectively. Its validation set consist of $20$ instances, starting with $12$ satellites, $24$ instruments and modes, and $48$ directions, and finishing with $31$ satellites, $62$ instruments and modes and $124$ directions. {\em Sieve of Erathostenes} requires $16$ lines and $3$ pointers to classify either as prime or non-prime, all the numbers comprised in the first $111$ natural numbers. {\em Spanner}, uses $14$ lines and $5$ pointers to solve all $20$ instances of the validation set, that range from $18$ spanners and nuts and a corridor with $14$ locations, to $46$ spanners and nuts and a corridor with $33$ locations.

\begin{table}[hbt!]
	\centering
	\begin{small}
		\begin{tabular}{|l|c|r||r|r||r|r|}\hline
			\textbf{Domain}  & $n,|Z|$ & \textbf{Instances} & Time$_\infty$ & Mem$_\infty$ & Time & Mem \\\hline
			Blocks Ontable & 13, 3 & 20 & 4.68 & 42MB & \textbf{0.46} & \textbf{5MB} \\
			Grid & 14, 2 & 248 & 4.49 & 0.2GB & \textbf{1.01} & {\bf 0.2GB} \\
			Miconic & 20, 3 & 20 & 4.00 & 45MB & \textbf{0.08} & \textbf{6.1MB} \\
			Michalski Trains & 15, 6 & 1 & 0.04 & 6.9MB & \textbf{0.00} & \textbf{4.7MB}\\
			Satellite & 43, 5 & 20 & ME & ME & \textbf{177.67} & \textbf{24MB} \\
			Sieve of Erathostenes & 16, 3 & 100 & 5.26 & 16MB & \textbf{0.55} & \textbf{8.8MB} \\
			Spanner & 14, 5 & 20 & 0.91 & 13MB & \textbf{0.04} & \textbf{5.7MB} \\\hline
		\end{tabular}
	\end{small}
	\caption{Validation of complex domains, CPU-time (secs) and memory peak for program validation, with/out {\em infinite program} detection. ME stands for memory exceeded. Best results in bold.}
	\label{tab:complex-validation}
\end{table}

\section{Conclusions}
\label{sec:conclusions}
We believe this work is a step-forward towards building a stronger connection between the areas of {\em automated planning} and programming. The work presented a formalization of classical planning as a vector transformation task, which is a common programming task. According to this formalism, computing a sequential plan for this tasks is computing a composition of vector transformation operations. Likewise computing a generalized plan is computing an algorithmic expression of the vector transformations. With the aim of building more bridges between automated planning and programming, we are exploring the extension of our approach to GP problems that include {\em real} state variables. We believe that we can address this kind of GP problems by introducing the notion of {\em precision} for the comparison of real variables, and redefining accordingly our mechanism for the update of the {\sc FLAGs} registers. 

Another interesting research direction is the extension of our {\em GP as heuristic search} approach for computing generalized plans starting from different input settings. For instance, the computation of generalized plans from a set of {\em plan traces} that demonstrates how to solve several planning problems. We are also interested on exploring the application of our {\em GP as heuristic search} approach to planning problems that are not goal-oriented, where the objective is to maximize a given {\em utility function}~\cite{lipovetzky2015classical}. In this particular setting, ideas from {\em approximated policy iteration}~\cite{bertsekas2011approximate}, and {\em reinforcement learning}~\cite{sutton2018reinforcement}, could be incorporated to our framework. On the other hand, the {\sc BFGP} algorithm starts from the empty program, but nothing prevents us from starting search from a partially specified generalized plan~\cite{bonet:sketches:21} with the aim of developing online approaches to GP. In fact, local search approaches have already shown successful for planning~\cite{gerevini2003planning} and program synthesis~\cite{solar2009sketching,gulwani2017program}. 

Our {\em cost-to-go heuristics} are still less informed than the current heuristics for classical planning, in the sense that our heuristics only consider the goals that are explicitly provided in the problem representation. A clear example is $f_5(\Pi,P_t)$, that builds on top of the {\em Euclidean distance}, and that for \strips\ planning problems is actually a goal counter. We believe that better estimates may be obtained by building on top of the powerful ideas of modern planning heuristics~\cite{hoffmann2003metric,Helmert:FD:JAIR06,frances2017effective}. In more detail, a promising approach for the development of more informative heuristics for GP is to consider sub-goals, that are not explicit given in the problem representation~\cite{hoffmann2004ordered}. For instance sets of sub-goals can be discovered as a pre-processing step, without grounding, regarding the set of {\em relevant} atoms that are traversed by the polynomial {\sc IW(1)} algorithm, when achieving individual goals~\cite{frances2017purely}. 

Since we are approaching GP as a classic tree search, a wide landscape of effective techniques, coming from {\em heuristic search} and {\em classical planning}, can actually improve the base performance of our approach. We mention some of the more promising ones. Large open lists can be handled more effectively  splitting them in several smaller lists~\cite{Helmert:FD:JAIR06}. {\em Delayed duplicate detection} could be implemented to manage large closed lists with magnetic disk memory~\cite{korf2008linear}. Further, more sophisticated mechanism can be implemented for handling closed nodes. For instance, once a search node is cancelled (e.g.~because $f_i(\Pi,\mathcal{P})$ identified that the {\em planning program} fails on a given instance), any program equivalent to this node should also be cancelled, e.g.~any program that can be built with transpositions of the causally-independent instructions. Given that the dept of the search-tree is bounded, techniques coming from SAT/CSP/SMTs, such a {\em non-chronological backtracking}, {\em limited discrepancy search}~\cite{korf1996improved}, or {\em taboo search}~\cite{nowicki1996fast}, might also result effective to improve our approach. Last, SATPLAN planners exploit multiple-thread computing to parallelize search in solution spaces with different bounds~\cite{rintanen2012planning}. This same idea could be applied to multiple searches for GP solutions with different program sizes.


\bibliography{gplanning}

\end{document}